\documentclass{article}

\usepackage{microtype}
\usepackage{graphicx}
\usepackage{subcaption}
\usepackage{booktabs} 
\usepackage{newclude} 

\usepackage[preprint]{icml2026}
\usepackage{hyperref}
\usepackage[table,xcdraw]{xcolor}

\usepackage{icml2026}

\usepackage{hyperref}
\hypersetup{
  bookmarks=true,
  bookmarksnumbered=true,
  bookmarksopen=true,
  bookmarksopenlevel=2,
  pdfstartview=FitH
}

\usepackage{amsmath}
\usepackage{amssymb}
\usepackage{mathtools}
\usepackage{amsthm}
\usepackage{tikz}
\usetikzlibrary{positioning, shapes, arrows.meta, fit, backgrounds}
\usepackage{algorithm}
\usepackage{algorithmic}
\renewcommand{\algorithmicrequire}{\textbf{Input:}}
\renewcommand{\algorithmicensure}{\textbf{Output:}}
\usepackage[capitalize,noabbrev]{cleveref}

\usepackage{multicol}
\usepackage{caption}
\usepackage{subcaption}

\theoremstyle{plain}
\newtheorem{theorem}{Theorem}[section]

\theoremstyle{definition}
\newtheorem{definition}[theorem]{Definition}

\theoremstyle{remark}
\newtheorem{remark}[theorem]{Remark}

\usepackage[textsize=tiny]{todonotes}

\icmltitlerunning{Provable Model Provenance Set  for Large Language Models}

\begin{document}

\twocolumn[
  \icmltitle{Provable Model Provenance Set for Large Language Models}



  \icmlsetsymbol{equal}{*}

  \begin{icmlauthorlist}
    \icmlauthor{Xiaoqi Qiu}{equal,yyy}
    \icmlauthor{Hao Zeng}{equal,yyy}
    \icmlauthor{Zhiyu Hou}{yyy}
    \icmlauthor{Hongxin Wei}{yyy}
   
  
  \end{icmlauthorlist}

  \icmlaffiliation{yyy}{Department of Statistics and Data Science, Southern University of Science and Technology, Shenzhen, China}

  \icmlcorrespondingauthor{Hongxin Wei}{weihx@sustech.edu.cn}

  \icmlkeywords{Model Provenance Testing, Large Language Model, Intellectual Property Protection}

  \vskip 0.3in
]



\printAffiliationsAndNotice{\icmlEqualContribution}

\begin{abstract}
The growing prevalence of unauthorized model usage and misattribution has increased the need for reliable model provenance analysis.
However, existing methods largely rely on heuristic fingerprint-matching rules that lack provable error control and often overlook the existence of multiple sources, leaving the reliability of their provenance claims unverified. 
In this work, we first formalize the model provenance problem with provable guarantees, requiring rigorous coverage of all true provenances at a prescribed confidence level.
Then, we propose the \emph{Model Provenance Set} (MPS), which employs a sequential test-and-exclusion procedure to adaptively construct a small set satisfying the guarantee.
The key idea of MPS is to test the significance of provenance existence within a candidate pool, thereby establishing a provable asymptotic guarantee at a user-specific confidence level.
Extensive experiments demonstrate that MPS effectively achieves target provenance coverage while strictly limiting the inclusion of unrelated models, and further reveal its potential for practical provenance analysis in attribution and auditing tasks. 



\end{abstract}

\section{Introduction}

Large language models (LLMs) are rapidly advancing and widely deployed across diverse applications~\cite{zhang2025llms}. 
The expansion of open-source ecosystems and API-based deployment has intensified concerns over attribution and unauthorized reuse, highlighting the importance of protecting LLM intellectual property (IP)~\cite{yao2024survey}. 
For example, prior reports~\cite{yoon2025intrinsic, openbmb2023minicpmv_issue196} have raised questions about the provenance of certain purportedly independent LLMs, with evidence suggesting that they are potentially fine-tuned from existing base models.
Such concerns motivate the need for reliable model provenance analysis to determine whether a model is derived from other released base models~\cite{shi2025knowledge}.

In the literature, existing research has revealed that models sharing a provenance relationship tend to exhibit consistent behavioral or structural patterns~\cite{nasery2025robust,zhang2025reef}. 
Accordingly, most methods infer the provenance by comparing signatures extracted from model behaviors or internal representations, typically via similarity-based heuristic rules such as thresholding~\cite{Nikolic2025Model,hu2025fingerprinting} or trained classifiers~\cite{zeng2024huref,wu2025llmdnatracingmodel}.
However, these methods lack provable error control, leading to unbounded false decisions that compromise the credibility of their provenance claims. 
Furthermore, they simply consider a single direct source and verify provenance through a winner-takes-all determination, overlooking the complex topologies where models may inherit multiple provenance via fine-tuning chains or merging~\cite{wu2025llmdnatracingmodel}. 
Collectively, these limitations necessitate a principled framework that provides provable guarantees for multi-model provenance, ensuring reliable and credible determinations in complex derivation lineages.

In this paper, we first formalize model provenance as a model-set identification problem with provable error control, establishing the theoretical foundation for controlling false-attribution rates at desired levels in multi-source scenarios. 
To this end, we propose \emph{Model Provenance Set} (MPS), a novel method that leverages a principled statistical hypothesis-testing framework to identify all true related models. 
The core idea of MPS is to statistically test the existence of a potential provenance model that exhibits significantly higher similarity to a target model than expected. 
In particular, MPS constructs a small yet valid provenance set through a test-and-exclusion procedure that iteratively checks for the existence of significantly similar models and excludes them until no distinguishably similar models remain. 
Accordingly, we establish the asymptotic theory of MPS to provide a provable coverage guarantee, validating the resulting set contains all true provenance with a pre-defined confidence level. In addition, MPS is agnostic to the choice of model similarity scores, enabling seamless integration with diverse fingerprinting techniques.

We evaluate MPS on a real-world benchmark of 455 LLMs from Hugging Face~\cite{HuggingFace2023}, spanning model sizes from 135M to 3B and lineage depth up to three. 
Our experiments demonstrate that MPS reliably identifies true provenance models while excluding unrelated ones. 
For example, with two true related models among 50 candidates, MPS achieves a coverage rate of 96\% with an average predicted set size of 2.1 at $\alpha$ = 0.05~(Figure~\ref{fig:main_res}), 
indicating its effective control of false provenances.
Beyond statistical guarantee analysis, we evaluate MPS across diverse provenance tasks, including model attribution, unauthorized derivation detection, and non-infringement claim assessment. 
The results show that MPS effectively distinguishes true provenances from irrelevant models with superior accuracy compared to empirical baselines, while providing rigorous significance metrics that serve as verifiable evidence for defending models against false infringement claims.
These findings underscore the potential of MPS as a practical and reliable tool for \textit{model IP protection} and \textit{auditing}.
Our code and data are available at \url{https://anonymous.4open.science/r/MPS-DFA7}.

We summarize our contributions as follows:
\begin{enumerate}
\item We formulate model provenance analysis as a principled model-set identification problem: identifying a provable provenance set that includes all true related models with at least a predefined confidence level.

\item We propose MPS to detect a compact provenance set with rigorous coverage guarantee via a sequential test-and-exclusion procedure.
Our method provides a provable asymptotic guarantee for its statistical validity while remaining fingerprint-agnostic.




\item Extensive experiments demonstrate that MPS consistently identifies all true provenance models with guaranteed probability while keeping the selected set compact, and enables reliable provenance analysis in diverse model attribution and auditing tasks.

\end{enumerate}




\begin{figure*}
    \centering
    \includegraphics[width=0.9\linewidth]{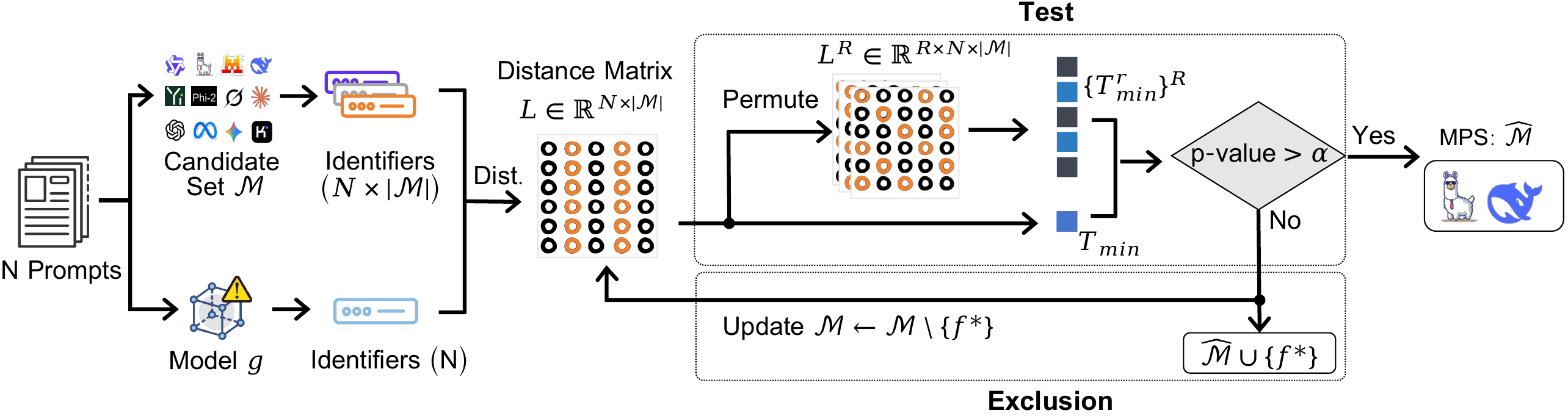}
    \caption{\textbf{Overview of the model provenance set framework.} 
Given a target model $g$ and candidate models $\mathcal{M}$, our MPS implements a sequential test-and-exclusion procedure to build a provenance set $\hat{\mathcal{M}}$ satisfying $Pr(\mathcal{M}^*\subseteq \hat{\mathcal{M}}) \ge 1-\alpha$, where $\mathcal{M}^*$ is the true provenance set. Orange circles in \textbf{$L$} indicate the distances between $g$ and true provenances, and $f^*$ is the most similar candidate model.}

    \label{fig:overview}
\end{figure*}

\section{Problem Formulation}
\label{sec:preliminary}
We first formulate the model provenance problem with theoretical guarantees. This provides a statistical foundation that moves beyond empirical heuristics to enable provable inferences. 
Such rigor is essential for grounding reliable provenance claims, yet it remains largely overlooked in prior literature; a detailed review is provided in Appendix~\ref{apd:related_work}. Below, we establish the necessary notations and definitions.
\paragraph{Provenance relationship.} 
Let $g$ be a target model and $\mathcal{M} = \{f_1, \dots, f_M\}$ be a set of candidate base models, which contains a mixture of LLMs, i.e., some are truly related to $g$, while others are not. We define the provenance relationship between two models as follows.

\begin{definition}[Provenance relationship]
\label{def:provenance}
A candidate model $f \in \mathcal{M}$ is said to be a \textit{provenance} model of the model $g$ if $g$ is obtained from $f$ through lighter customizations, including but not limited to fine-tuning, quantization, or architectural modifications (e.g., mixture-of-experts).
\end{definition}

\paragraph{Model similarity score.}
Typically, the provenance relationship is measured by model-unique fingerprint similarity over a query distribution $\mathcal{Q}$.
Given $N$ prompts, i.e., $X$ = $\{x_1,\dots,x_N\}\sim\mathcal{Q}$, let $f(x)$ denote the identifier of model $f$ to query $x$ (e.g., outputs or embeddings).
In this paper, we examine the deviation between the resulting traces of a candidate model $f_i\in \mathcal{M}$ and a target model $g$ over $X$, and define a sample-wise distance function $L_{i,t}$ to measure the dissimilarity score between $f_i(x_t)$ and $g(x_t)$ on prompt $x_t$.
\begin{equation}
    \label{eq:loss}
    L_{i,t} = 1 - s(f_i(x_t), g(x_t)),
\end{equation}
where $s(\cdot,\cdot) \in [0, 1]$ returns a constant similarity score. We then compute the average distance on the query space, i.e., $\mu_i = \mathbb{E}[L_{i,t}]$. A higher $\mu_i$ implies a lower similarity between model $f_i$ and the target model, and vice versa.

\paragraph{Model provenance set.} Among all the candidate models, we define the ones related to $g$ as ``true provenance set'':
\begin{definition}[True provenance set]
    \label{definition:tam}
   The subset of candidate models that hold a provenance relationship to the target model $g$ is the \textit{true provenance set}, denoted as $\mathcal{M}^*$.
\end{definition}

\begin{remark}
By definition, $\mathcal{M}^* \subseteq \mathcal{M}$, with no assumption that $\mathcal{M}^*$ is non-empty, i.e., $|\mathcal{M}^*| \ge 0$. 
When $\mathcal{M}^* = \emptyset$, all candidate models are unrelated to the model $g$.
When $|\mathcal{M}^*| > 1$, there exist multiple provenance models that typically  exhibit statistically indistinguishable dissimilarity to $g$, or show a progressively decreasing trend along the lineage, depending on the underlying derivation mechanisms.
\end{remark}

\paragraph{Objective.}
Our goal is to identify a minimal model provenance set that includes all true provenances with high probability. Formally, we seek to construct a subset $\hat{\mathcal{M}} \subseteq \mathcal{M}$ that solves the following optimization problem:
\begin{equation}
\label{eq:objective}
\min_{\hat{\mathcal{M}} \subseteq \mathcal{M}} |\hat{\mathcal{M}}| \quad \text{subject to} \quad \Pr(\mathcal{M}^* \subseteq \hat{\mathcal{M}}) \ge 1 - \alpha,
\end{equation}
where $\alpha \in (0,1)$ denotes the maximum allowable error rate of missing any true provenance models.

The goal addresses the dual requirements of reliability and efficiency for provenance analysis. Specifically, bounding the error rate provides the statistical safety that no true provenance is missed, while minimizing the set size ensures to effectively eliminate the irrelevant candidates. 
This formulation is valuable in model auditing or management scenarios, where provenance tools typically serve as a preliminary screening stage followed by more resource-intensive verification~\cite{nasery2025robust,Nikolic2025Model}, such as manual expert checking. 
Thus, achieving this objective helps alleviate the burden of such downstream verification.
To our knowledge, this work first introduces a formal framework with provable guarantees for model provenance.
\section{Methodology}
\label{sec:method}

Identifying provenance sets in practice is challenged by the absence of a clear similarity boundary to confirm provenance relationships.
Moreover, potential discrepancies in similarities among multiple true sources complicate comprehensive identification.
Inspired by the Model Confidence Set (MCS)~\citep{Hansen2011mcs}, which identifies a superior model set via significance testing,
we address these issues by designing a relative-distance-based test statistic to detect provenance
existence and introducing a sequential test-and-exclusion procedure that selects one most related model at each step.
Our overall framework is presented in Figure~\ref{fig:overview}, with detailed connection to MCS in Appendix~\ref{apd:mcs_connection}.



\subsection{The design of test statistic }
\label{sec:test_stat}
This subsection details how to build the test statistic. Let $L \in \mathbb{R}^{N \times |\mathcal{M}|}$ denote the distance matrix comprising the sample-wise dissimilarity of all $f_i \in \mathcal{M}$ with the target model $g$.
For any model $f_i$ in the current candidate set $\mathcal{M}$, we compute the relative dissimilarity score $d_{ij,t} = L_{i,t} - L_{j,t}$ and its average on the query space $\bar{d}_{ij} = \frac{1}{N} \sum_{t=1}^{N} d_{ij,t}$.
The relative deviation of model $f_i$ from ensemble-average distances to the target $g$ is defined as follows: 
\begin{equation}
\bar{d}_{i \cdot} = \frac{1}{M} \sum_{j \in \mathcal{M}} \bar{d}_{ij},
\end{equation}
where $M$ is the size of the candidate set. Note that $\bar{d}_{i\cdot}\in \mathbb{R}$ can be seen from the identity $\bar{d}_{i\cdot}=(\bar{L}_i-\bar{L}_{\cdot})$, where $\bar{L}_i \equiv  N^{-1}\sum_{t=1}^NL_{i,t}$ and $\bar{L}_{\cdot} \equiv  M^{-1}\sum_{i=1}^M \bar{L}_i$.
This leads to the studentized $t$-statistic:
\begin{equation}
\label{eq:t_stat}
t_{i} = \frac{\bar{d}_{i \cdot}}{\sqrt{\widehat{\mathrm{var}}(\bar{d}_{i \cdot})}},
\end{equation}
where $\widehat{\mathrm{var}}(\bar{d}_{i \cdot}) \equiv \frac{1}{N-1}\sum_{t=1}^N(d_{i\cdot ,t}-\bar{d}_{i\cdot } )^2 $  is the standard sample variance of the observations $\{d_{i,t}\}_{t=1}^N$.

Let $u_i \equiv \bar{d_i}$, the null hypothesis for the candidate model set $\mathcal{M}$ states that all candidates have equal expected distance relative to the target model $g$, formulated as:
\begin{equation}
\label{eq:null}
\mathrm{H}_{0, \mathcal{M}}: \mu_{i} = \mu_{j} \quad \forall f_i, f_j \in \mathcal{M}.
\end{equation}

We denote the alternative hypothesis, $u_i\ne u_j$ for some $f_i,f_j\in \mathcal{M}$, by $H_{A,\mathcal{M}}$. We test this using the $T_{\min}$ statistic:

\begin{equation}
\label{eq:tmin}
T_{\min, \mathcal{M}} = \min_{f_i \in \mathcal{M}} t_{i},
\end{equation}
which identifies the model most similar to the target (lowest relative distance).
$T_{\min}$ reflects the smallest relative deviation in distance to the target $g$ among all candidates.
Testing whether such a value is significantly small allows us to determine whether $\mathcal{M}$ contains models that are anomalously close to $g$, which are likely to be provenance models.

\subsection{Hypothesis testing using $t$-statistics}
\label{sec:permutation}
We assess the statistical significance of $T_{\min}$ via a $p$-value. 
Specifically, we use a permutation test to approximate the null distribution under the hypothesis of equal closeness to $g$. 
Under $\mathrm{H}_{0,\mathcal{M}}$, the discrepancy scores are exchangeable across candidate models; equivalently, all candidate-to-$g$ distances share the same statistical properties, and permuting these values does not affect the null hypothesis.


We describe the permutation procedure as follows.
For each prompt $x_t$, the score values in the vector $L_{\cdot, t} = (L_{1,t}, \dots, L_{M,t})$ are randomly permuted across all models.
This allows the formation of a null distribution by breaking systematic model differences, helping to check whether provenance signals are merely artifacts of random coincidence. Specifically, we recompute all $t$-statistics and obtain a new $T_{\min}^{(r)}$ after permuting the whole distance matrix.
Repeating this process $R$ times, we obtain a set of permuted statistics $\{T_{\min}^{(r)}\}_{r=1}^R$.
The $p$-value is computed as:
\begin{equation}
\label{eq:pvalue}
p_{\mathcal{M}} = \frac{1}{R} \sum_{r=1}^{R} \mathbf{1}\Big\{ T_{\min}^{(r)} \le T_{\min}^{\mathrm{obs}} \Big\},
\end{equation}
which quantifies the likelihood of $T_{\mathrm{min}}^{obs}$ under $\mathrm{H}_{0,\mathcal{M}}$ and is typically compared with a predetermined significance level $\alpha$ to determine the statistical significance of the result.

Following this criterion, we reject the $H_{0,\mathcal{M}}$ whenever $p_{\mathcal{M}} \le \alpha$, indicating that the observed minimal distance is statistically significant (i.e., too small to be explained by random chance).
Otherwise, when $p_{\mathcal{M}} > \alpha$, we fail to reject $\mathrm{H}_{0,\mathcal{M}}$ and conclude that no remaining model can be statistically distinguished as being closer to $g$. 



\begin{theorem}[Valid $p$-value]
\label{thm:pvalue}
For any model $f_i, f_j \in \mathcal{M}$, suppose that (a) prompts $\{x_t\}_{t=1}^N$ are sampled independently, (b) $\mathbb{E}[|d_{ij,t}|^{2+\delta}] < \infty$ for some $\delta > 0$, and (c) $\sigma^2_{ij} = \mathrm{Var}(d_{ij,t}) > 0$ whenever $\mu_i = \mu_j$.
Then the permutation-based $p$-value $p_{{\mathcal{M}}}$ computed in Eq.~\eqref{eq:pvalue} is asymptotically valid.
That is, under the null hypothesis,
\begin{equation}
\lim_{N, R \to \infty} \Pr(p_{{\mathcal{M}}} \le \alpha ) \le \alpha.
\end{equation}
\end{theorem}
Condition (a) ensures independence across samples, (b) ensures the central limit theorem applies, and (c) rules out degenerate cases where two models produce identical outputs. 
These are regular conditions applied to testing problems.
The proof of Theorem~\ref{thm:pvalue} appears in Appendix~\ref{prf:pvalue}.

\subsection{Sequential test-and-exclusion algorithm}
\label{sec:algorithm}

To select the provenance models, we iteratively refine the set $\mathcal{M}$, as detailed in Algorithm~\ref{alg:mps}.
In each iteration, if the models in $\mathcal{M}$ are indistinguishable in their similarity to the target model $g$, i.e., $p_{\mathcal{M}}\le \alpha$, we reject this null hypothesis and exclude the most similar model, identified by $\arg\min_{f_i \in \mathcal{M}} t_i$. 
This procedure terminates when the equivalence test accepts, indicating that the remaining models are statistically indistinguishable in their closeness to $g$.
The final predicted provenance set $\hat{\mathcal{M}}$ consists of all excluded models during the procedure.
In this framework, MPS yields a small predicted set that preserves all true provenance models with a confidence level of at least $1-\alpha$, a result that fulfills the dual objectives defined in Eq.~\eqref{eq:objective}.

\begin{algorithm}[ht!]
\caption{Model provenance set (MPS)}
\label{alg:mps}
\begin{algorithmic}
\REQUIRE Candidate set $\mathcal{M}_0 = \{f_1, \dots, f_M\}$, model $g$, prompts $X = \{x_1, \dots, x_N\}$, significance level $\alpha$, permutation rounds $R$, and a distance function.
\ENSURE Predicted provenance set $\hat{\mathcal{M}}$.

\STATE Compute distances matrix $L \in \mathbb{R}^{N \times |\mathcal{M}|}$ by Eq.~(\ref{eq:loss})
\STATE $\mathcal{M} \leftarrow \mathcal{M}_0$, $\hat{\mathcal{M}} \leftarrow \emptyset$.
\WHILE{$|{\mathcal{M}}| > 1$}
    \STATE Compute $t_i$ for all $f_i \in  {\mathcal{M}}$ using Eq.~\eqref{eq:t_stat}.
    \STATE $T_{\min}^{\mathrm{obs}} \leftarrow \min_{f_i \in {\mathcal{M}}} t_i$.
    \FOR{$r = 1$ to $R$}
         \STATE  Permutation$(L_{1,t}, \dots, L_{|{\mathcal{M}}|,t})$ for $t=1\cdots N$
        \STATE Calculate $t_i^{(r)}$ and $T_{\min}^{(r)} \leftarrow \min_{f_i \in {\mathcal{M}}} t_i^{(r)}$.
    \ENDFOR
    \STATE $p_{{\mathcal{M}}} \leftarrow \frac{1}{R} \sum_{r=1}^{R} \mathbf{1}(T_{\min}^{(r)} \le T_{\min}^{\mathrm{obs}})$.
    \IF{$p_{{\mathcal{M}}} > \alpha$}
        \STATE \textbf{break}
    \ENDIF
    \STATE $f^* \leftarrow \arg\min_{f_i \in {\mathcal{M}}} t_i$.
    \STATE ${\mathcal{M}} \leftarrow {\mathcal{M}} \setminus \{f^*\}$, $\hat{\mathcal{M}} \leftarrow \hat{\mathcal{M}} \cup \{f^*\}$.
\ENDWHILE
\end{algorithmic}
\end{algorithm}

To formally justify these properties,  we now establish the theoretical guarantees for Algorithm~\ref{alg:mps}.




\begin{theorem}[Coverage guarantee]\label{thm:coverage}
Under the conditions of Theorem~\ref{thm:pvalue}, Algorithm~\ref{alg:mps} with significance level $\alpha$ produces a provenance set $\hat{\mathcal{M}}$ such that
\begin{equation}
\lim_{N,R \to \infty} \Pr\left( \mathcal{M}^* \subseteq \hat{\mathcal{M}} \right) \ge 1 - \alpha.
\end{equation}
\end{theorem}

\begin{theorem}[Asymptotic efficiency]
\label{thm:efficiency}
Suppose there exists a gap $\delta > 0$ such that
\[
\max_{f \in \mathcal{M}^*} \mu_f < \min_{f \in \mathcal{M} \setminus \mathcal{M}^*} \mu_f - \delta.
\]
Then Algorithm~\ref{alg:mps} produces $\hat{\mathcal{M}} = \mathcal{M}^*$ with probability
\[
\Pr(\hat{\mathcal{M}} = \mathcal{M}^*) \ge 1 - O(\exp(-c N \delta^2))
\]
for some constant $c > 0$.
\end{theorem}

\begin{remark}[Minimum detectable gap]
\label{rmk:min-gap}
For reliable detection, the gap $\delta$ should exceed the noise level, i.e., $\delta \gtrsim \bar{\sigma} / \sqrt{N}$ where $\bar{\sigma} = \inf_{i \ne j} \mathrm{Var}(d_{ij,t})$.
This implies that the minimum detectable gap scales as $\delta_{\min} = O(N^{-1/2})$, which matches the standard parametric rate.
Equivalently, to detect a gap of size $\delta$, the required sample size is $N = O(\bar{\sigma}^2 / \delta^2)$.
\end{remark}

Theorem~\ref{thm:coverage} ensures that MPS satisfies the coverage constraint in the Eq.~\eqref{eq:objective}, while Theorem~\ref{thm:efficiency} shows that MPS recovers the exact provenance set $\mathcal{M}^*$ asymptotically.
The proofs are deferred to Appendix~\ref{prf:coverage} and Appendix~\ref{prf:efficiency}.

\section{Experiments}

\begin{figure*}[tb]
    \centering
    \includegraphics[width=\linewidth]{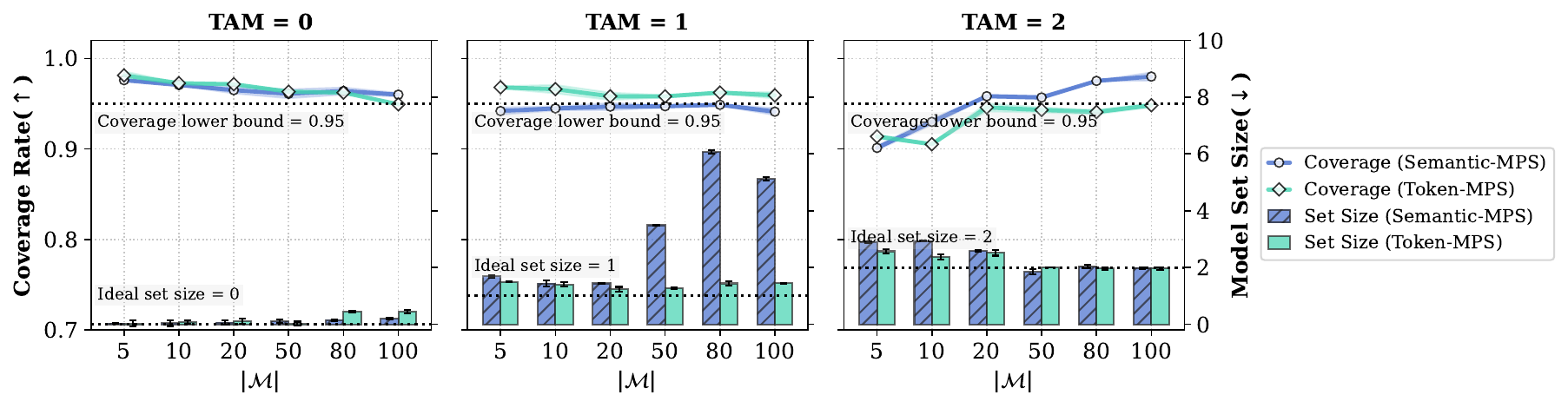}
    \caption{Provenance performance of Semantic-MPS and Token-MPS at significance level $\alpha=0.05$. $|\mathcal{M}|$ denotes the size of candidate sets, and $\mathrm{TAM}$ denotes the number of true ancestor models in one candidate set.}
    \label{fig:main_res}
\end{figure*}

\begin{figure*}[tb]
    \centering
    \includegraphics[width=\linewidth]{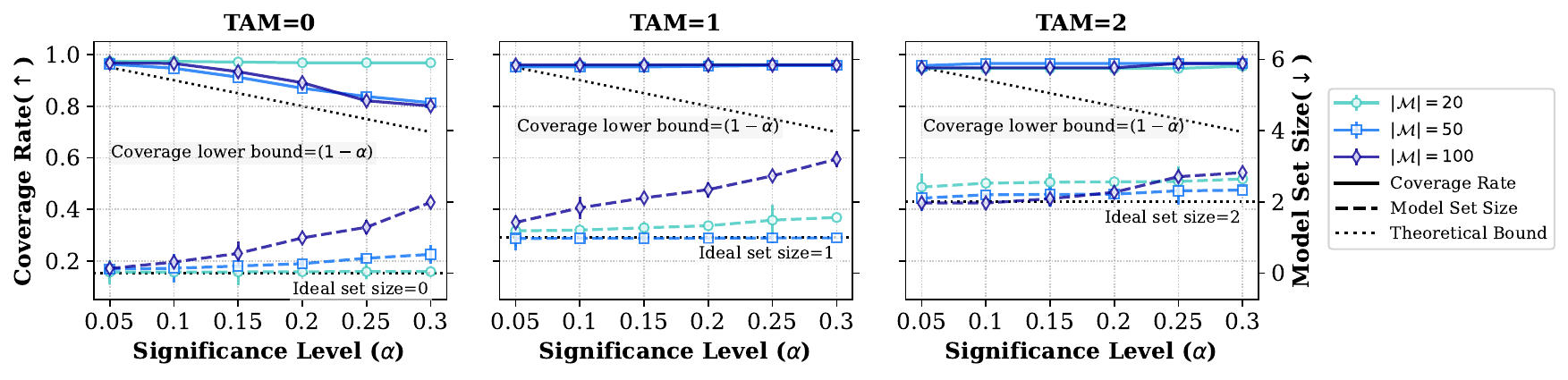}
     \caption{
     Token-MPS performance across varying $\alpha$. $|\mathcal{M}|$ denotes the candidate set size; $\mathrm{TAM}$ is the number of true ancestor models.}
     
     \label{fig:a_01_res}
\end{figure*}

\begin{table}[t]
\caption{Benchmark statistics stratified by lineage depth.
For $\mathrm{TAM} \!$ = 0/1, fewer derived models than derivation chains arise because a model can participate in multiple derivation chains, serving as both an upstream child and a downstream parent. ``-'' indicates settings omitted due to insufficient number of models.
}
\label{tab:mptbench}
\centering
\begin{tabular}{@{}llll@{}}
\toprule
\textbf{Feature}     &  \textbf{TAM=1/0}    &  \textbf{TAM=2} &  \textbf{TAM=3}   \\ \midrule
Pre-trained Models   & 78             &  38  &  9  \\
Derived Models       & 377            & 118  &  13 \\
Total Models         & 455            & 145 &  22 \\
Derivation Chains & 383           & 118  &  13 \\ 
\#Test dataset     & 383        & 118  &  -  \\
\bottomrule
\end{tabular}
\end{table}

\label{experiments}
In this section, we empirically validate the theoretical guarantees of MPS on a real-world benchmark with multi-model provenance. Specifically, we first evaluate MPS under various distance functions across various candidate pool scales. Then, we analyze how the MPS procedure is impacted by the significance level $\alpha$, permutation count $R$, and the number of prompts $T$, respectively. 

\begin{figure*}[tb]
    \centering
    \includegraphics[width=\linewidth]{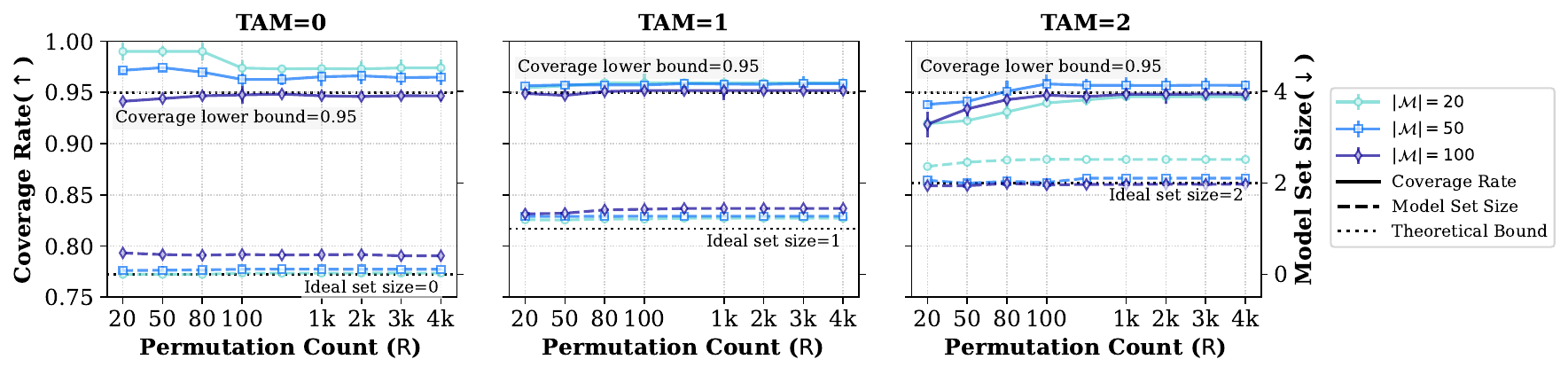}
    \caption{Token-MPS performance across varying $R$. $|\mathcal{M}|$ denotes the candidate set size; $\mathrm{TAM}$ is the number of true ancestor models.}
    \label{fig:varying_r}
\end{figure*}

\begin{figure*}[tb]
    \centering
    \includegraphics[width=\linewidth]{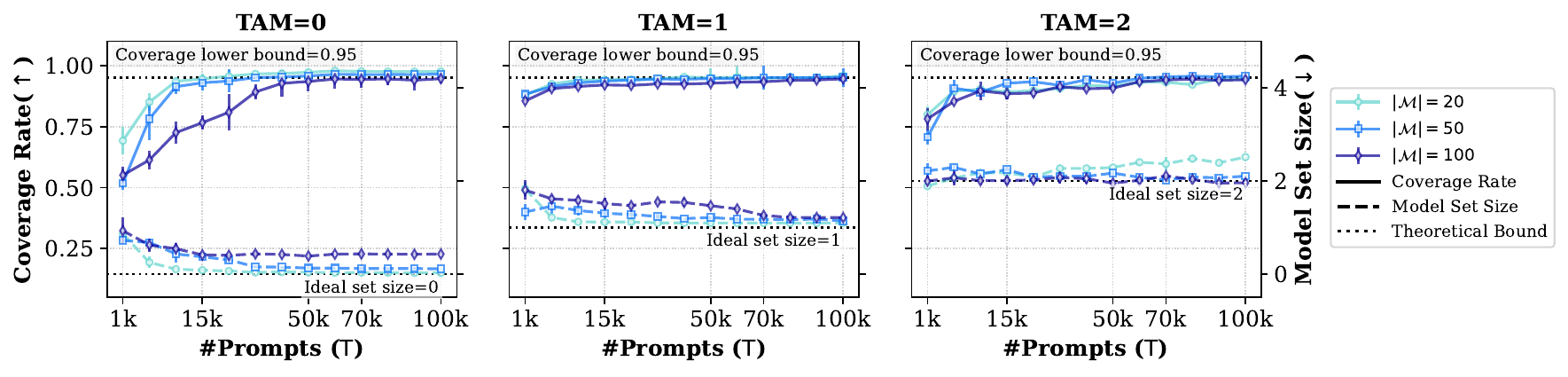}
    \caption{Token-MPS performance across varying $T$. $|\mathcal{M}|$ denotes the candidate set size; $\mathrm{TAM}$ is the number of true ancestor models.}
    \label{fig:varying_prompt}
\end{figure*}

\subsection{Experimental setup}
\textbf{Models and benchmark.}
We collect widely used LLMs based on download counts from the Hugging Face platform~\cite{HuggingFace2023}. After filtering out models without explicit derivation links and those exceeding 3B parameters, we obtain 455 LLMs with lineage depth up to three.
Models are then grouped by their number of \emph{True Ancestor Models} (TAM). 
For example, given a derivation chain A $\rightarrow$ B $\rightarrow$ C  $\rightarrow$ D, we have $\mathrm{TAM}(D) = 3$.
Based on established derivation chains, we construct test instances with varying $\mathrm{TAM}$ count, each consisting of a target model and a candidate set of size $|\mathcal{M}|$ that includes all true ancestors and $|\mathcal{M}|-\mathrm{TAM}$ randomly sampled unrelated models.
Test instances in $\mathrm{TAM}$ = 0 are derived from $\mathrm{TAM}$ = 1 chains by randomly sampling $|\mathcal{M}|$ unrelated models per chain. Note that we omit the case of $\mathrm{TAM}$ = 3  due to insufficient samples.
Table~\ref{tab:mptbench} summarizes the benchmark, with details on model collection and test construction in Appendix~\ref{apd:benchmark}.

\textbf{Framework.} We implement two established indicators to measure provenance relationships: (1)~next-token distance \cite{Nikolic2025Model} for linguistic divergence, denoted as \textit{Token-MPS}, and (2)~response semantic distance \cite{wu2025llmdnatracingmodel} for semantic dissimilarity, denoted as \textit{Semantic-MPS}. We employ ten popular closed-source LLMs to generate 100k diverse incomplete sentences as prompts to elicit behavioral patterns. 
More details of distance functions and prompts collection are described in \textcolor{black}{Appendix}~\ref{apd:prompt_con}.

\textbf{Evaluation metrics.} 
We evaluate the performance of the MPS procedure using two typical metrics for set-valued inference~\cite{angelopoulos2021gentle}: (1) \textbf{Coverage rate}, which measures the fraction of instances whose true ancestor models are completely identified. By design, it is lower bounded by  $1 - \alpha$.
(2) \textbf{Predicted set size}, which is closer to ground-truth size indicating better selection.

\textbf{Implementation details.}
In the MPS procedure, each model is asked to predict the next token via greedy decoding, ensuring stable behavioral patterns to suppress stochastic noise.
All experiments are conducted with a significance level of $\alpha\!=\!0.05$, permutation count $R\!=\!1000$ , and prompt count $T\!=\!100\mathrm{k}$, unless otherwise stated. We report the mean values of each metric across 10 runs. 

\subsection{Main results}

\textbf{MPS enables coverage of all true provenances with near-minimal sets. }
Figure~\ref{fig:main_res} shows the averaged performance of MPS across varying $|\mathcal{M}|\in [5,100]$ at significance level $\alpha$ = 0.05, with results for other values of $\alpha$ reported in Appendix~\ref{apd:main_res_alpha}.
A salient observation is that MPS under both distance functions consistently maintains coverage rates above the desired confidence bound $1\!-\!\alpha$ across varying $|\mathcal{M}|$. For example, even in a challenging setting with 100 candidates, MPS can robustly cover all sources with high probability.
Note that this robustness does not lead to overly conservative (large) sets; the predicted set size stays tightly concentrated around the ideal value even as the candidate pool $\mathcal{M}$ expands.
Overall, MPS exhibits robust performance across varying candidate pool scales by precisely identifying the true source while eliminating irrelevant models.

\textbf{MPS benefits from larger candidate pools in deep lineage settings.} 
As $|\mathcal{M}|$ increases, Figure~\ref{fig:main_res} reveals an interesting divergence: performance decays for  $\mathrm{TAM} \! \le 1$ with lower coverage and larger sets, yet improves for deep lineages ($\mathrm{TAM} \! =\!2$) in both metrics. 
We attribute this shift to the lineage signal attenuation at greater depths (see Appendix~\ref{apd:score distribution analysis} for detailed analysis). 
Specifically, in direct derivations ($\mathrm{TAM} \! \in \{0,1\}$), a larger pool
increases the risk of encountering distractors that mimic the true sources (e.g., same-family models), leading to its accidental miss of the true one or more redundant selections. 
Conversely, for the weaker signals in deep lineages, a larger pool provides a more precise estimation of the null distribution. This enhances statistical resolution, allowing MPS to better distinguish subtle provenance signals from noise.
Such scale-driven refinement highlights the feasibility of MPS for large-scale model provenance auditing in complex lineage settings.

\begin{figure}[t]
    \centering
    \begin{subfigure}{\linewidth}
         \centering
        \includegraphics[width=\linewidth]{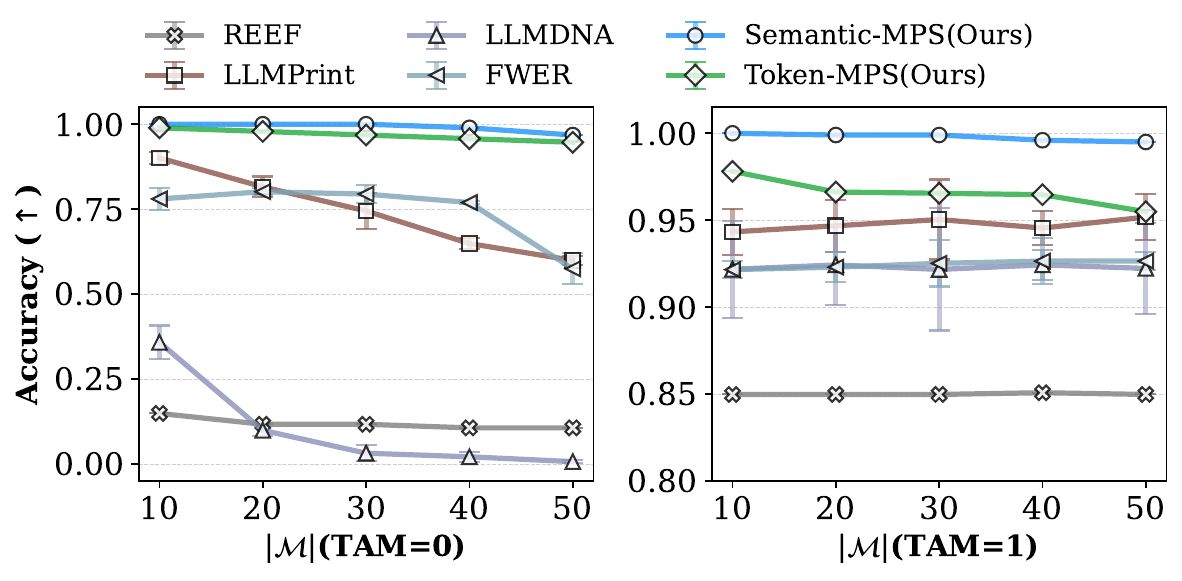}
    \end{subfigure}

     \begin{subfigure}{\linewidth}
        \centering
        \includegraphics[width=\linewidth]{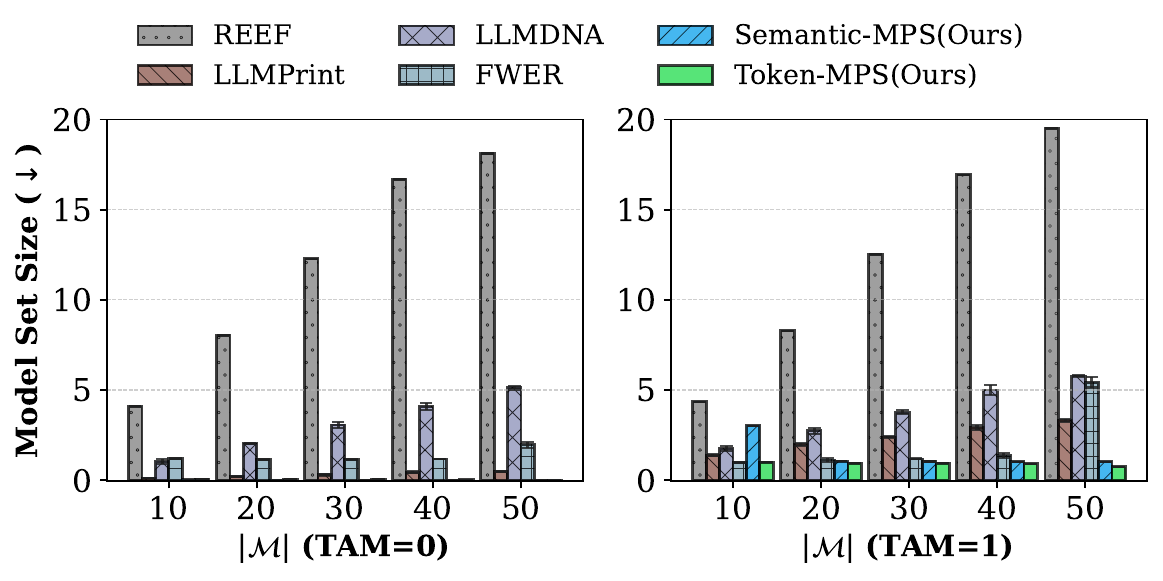}
    \end{subfigure}
    
    \caption{
      Provenance attribution performance comparison across different methods under multi-candidate settings. $|\mathcal{M}|$ denotes the candidate pool size; $\mathrm{TAM}$ is the number of true parent models.}
    \label{fig:mutil-souce mpts_res}
\end{figure}

\begin{figure}[t]
    \centering

    \includegraphics[width=\linewidth]{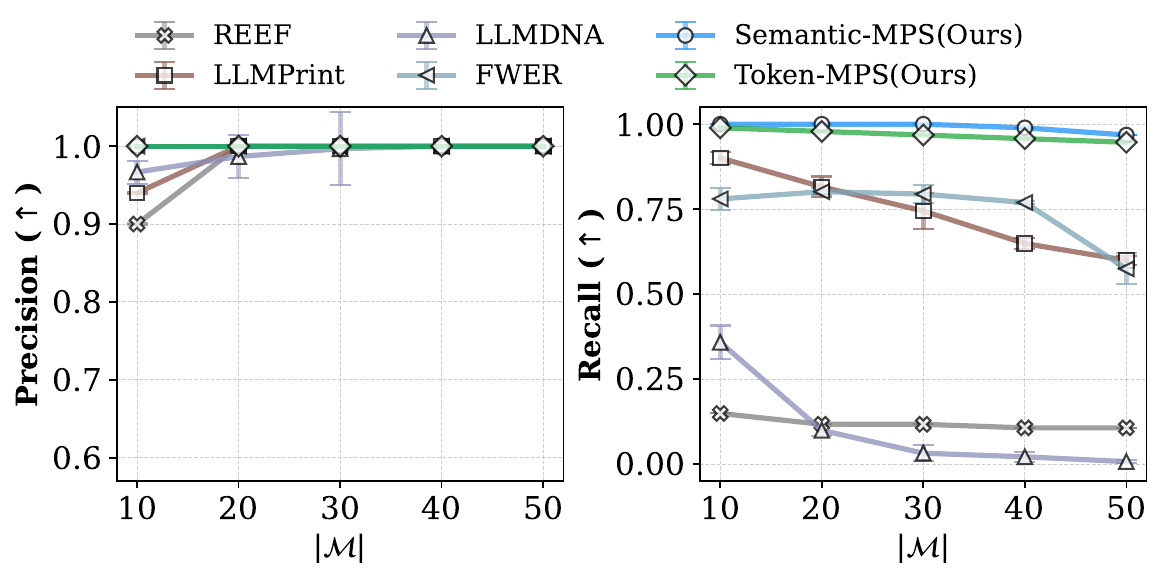}
    \caption{Unauthorized derivative model detection comparison across different methods under multi-candidate settings.}
    \label{fig:case1}
\end{figure}



\textbf{MPS maintains statistical validity across significance levels.}
We also verify that the empirical behavior of MPS conforms to its theoretical guarantees. As shown in Figure~\ref{fig:a_01_res} for Token-MPS and Appendix~\ref{apd:semantic_alpha_ana} for Semantic-MPS, the coverage consistently remains above the 1-$\alpha$ bound, demonstrating the statistical validity of MPS through rigorous error control. 
Specifically, a larger $\alpha$ relaxes this theoretical floor, providing more error tolerance for false selection. Then the coverage for $\mathrm{TAM} \!=\!0$ (favoring empty sets) decreases as expected, while for $\mathrm{TAM} \! > \! 0$ it improves by capturing near-top candidates which may be missed under stricter bounds.
Notably, even with a larger $\alpha$, the predicted set remains remarkably compact. For example, at $\alpha\!=\!0.3$ and $|\mathcal{M}|\!=\!100$, the set size stays below 4.
Overall, these results demonstrate that MPS exhibits well-calibrated behavior consistent with theory, allowing practitioners to flexibly tune $\alpha$ to meet the needs without compromising the statistical reliability.


\textbf{MPS achieves stability with modest computational overhead.}
To assess the cost-effectiveness of MPS, we evaluate Token-MPS across two different overhead settings: permutation counts ($R$) and prompt counts ($T$). As shown in Figure~\ref{fig:varying_r} and Figure~\ref{fig:varying_prompt} respectively,
the coverage of MPS exceeds the $1-\alpha$ bound even at the minimum tested scale, e.g., $R\!=\!20$ or $T\!=\!15\mathrm{k}$ when $\mathrm{TAM} \! =\! 1$, while the prediction set size rapidly shrinks toward the ideal value as overhead increases. 
Specifically, for permutation, both coverage rate and set size stabilize once $R\ge100$, even in a challenging $\mathrm{TAM} \! =\! 2$ setting with 100 candidates.
Regarding prompt counts, performance of MPS gradually aligns with the theoretical bound, reaching a plateau at a point between $T\!=\!15\mathrm{k}$ and $70\mathrm{k}$ that depends on $\mathrm{TAM}$ setting.
At lower $T$, insufficient samples fail to stably sustain provenance signals, thereby yielding sub-theoretical results.
Note that such prompts are effortless to acquire as they are randomly generated simple sentences.
Overall, MPS requires relatively small $R$ and $T$ to provide reliable determination, ensuring its practicality for large-scale provenance tasks.

\section{Empirical Applications}
\label{sec:emprical apps}
Beyond statistical guarantee evaluation, this section investigates the empirical applicability of MPS in model auditing scenarios. Specifically, we study its effectiveness for two downstream decision-making tasks and its use in providing quantitative evidence for model non-infringement.

\subsection{Model provenance attribution}
\label{subsec: provenance attribution}
Model provenance attribution aims to identify the specific source models of the target model, enabling analyses such as model tracing, intellectual property compliance, and development auditing. While prior works on model provenance largely focus on pairwise verification between a single suspected source and a target model~\cite{hu2025fingerprinting,wu2025llmdnatracingmodel}, many practical scenarios involve multiple candidate models and unknown derivation relationships. We therefore study the general provenance attribution case, where a target model is examined against multiple candidates and the goal is to correctly find the true parent models when they exist, or return an empty set when no such parents are present.

\textbf{Setup.}
We evaluate this task using Bench-A~\cite{Nikolic2025Model}, a human-curated dataset including 100 provenance and non-provenance pairs and 10 held-out reference models. We extend it to the multi-candidate setting by augmenting each suspect parent model with randomly selected models unrelated to $g$. We compare MPS with existing provenance detection methods, including white-box approaches such as REEF~\cite{zhang2025reef} and LLMPrint~\cite{hu2025fingerprinting}, and black-box techniques such as LLMDNA~\cite{wu2025llmdnatracingmodel} and FWER~\cite{Nikolic2025Model}. Each baseline is applied iteratively to each candidate–target pair, and candidates with positive decisions are aggregated into a predicted provenance set. All results are averaged over three runs, reporting attribution accuracy and predicted set size. More details of the baselines are provided in Appendix~\ref{apd: detail_pairwise_detection}.

\textbf{Results.}
Figure~\ref{fig:mutil-souce mpts_res} presents the averaged attribution accuracies and set size of different methods under multi-candidates scenario.
Compared to the empirical baselines, MPS consistently outperforms across all scales, maintaining higher accuracy (over 0.97 for Semantic-MPS and 0.95 for Token-MPS) while keeping smaller predicted set sizes near ground-truth levels.
Moreover, all baselines suffer obvious accuracy degradation in the no-parent setting ($\mathrm{TAM} \! =\!0$). This occurs because larger candidate sets inflate false-positive errors and produce incorrect non-empty outputs, as evidenced by the expansion of predicted set sizes.
By contrast, MPS remains relatively stable as $|\mathcal{M}|$ grows, since it jointly evaluates candidate models rather than making independent decisions, highlighting its robustness in large-scale provenance analysis.
We also evaluate our MPS on original pairwise instances in Bench-A, as reported in Appendix~\ref{apd:pairwise_dectection}, further demonstrating its effectiveness for model provenance attribution.

\subsection{Unauthorized derivative model detection}
During model onboarding, platforms or third-party assessors are primarily concerned with whether a to-be-deployed model exhibits potential IP infringement risks against existing models~\cite{yoon2025intrinsic,zhang2025self}, rather than identifying the specific source models. Motivated by this practical consideration, we employ MPS as a lightweight risk-aware decision tool to detect whether a given model is derived from any candidate models without authorization based on the predicted set size. 
Specifically, an empty set suggests the model is free from reuse, while a non-empty output flags potential derivation risk.

\textbf{Setup.}
We use the same Bench-A instances and baseline methods as in Section~\ref{subsec: provenance attribution}.
For clarity, we treat multi-candidate instances with no true parents as risk-free cases, i.e., $\mathrm{TAM}(g)=0$, and all others as risky, with the risky class treated as the positive class.
For each candidate scale, we construct a balanced benchmark with 100 risk-free and 100 risky test instances.
We apply MPS and all baselines to produce a predicted provenance set, and classify an instance as risk-free if the set is empty, and risky otherwise.
We report standard binary classification metrics, including Precision and Recall, averaged over three runs.

\textbf{Results.} 
The averaged performance of different detectors is presented in Figure~\ref{fig:case1}. In particular, our MPS outperforms baselines with both Precision and Recall consistently above 0.97 across all candidate settings, demonstrating its stable and reliable risk-aware decisions.
Although baseline methods maintain high Precision as the candidate set increases, their Recall declines sharply, indicating that many risk-free cases are incorrectly classified as risky. For instance, the Recall of REEF falls to around 0.1 when $|\mathcal{M}|=50$. 
This behavior arises from their any-hit decision strategy, in which a single false provenance can trigger the risky prediction, leading to an accumulation of false positives as the candidate pool grows. 
Conversely, MPS aggregates relevance signals from all candidates relative to the target model, enabling more reliable discrimination and stable performance across different candidate set sizes.


\subsection{Non-infringement claim assessment}
In model IP disputes, it is critical to quantify the sufficiency of a model non-infringement claim, ensuring the verdict is backed by rigorous evidence rather than mere heuristics~\cite{xu2025copyright}.
MPS naturally provides a principled assessment mechanism for this purpose by testing the null hypothesis $H_{0,\mathcal{M}}$, i.e., assuming the model $g$ is not derived from the existing models in the set $\mathcal{M}$. 
If the $p$-value associated with $H_{0,\mathcal{M}}$ is greater than the significance level $\alpha$, MPS fails to reject $H_{0,\mathcal{M}}$, indicating that model $g$ cannot be statistically linked to any model in $\mathcal{M}$. 
By design, when $\alpha= p_{\mathcal{M}}$, the resulting empty provenance set suggests insufficient evidence to support the infringement claim of $g$.
Consequently, we utilize the $p$-value $p_{\mathcal{M}}$ to provide evidence that model $g$ is not derived from any pre-existing model in $\mathcal{M}$, and refer to it as the \textbf{N}on-\textbf{I}nfringement \textbf{Score} (NI-Score). A higher NI-Score indicates that the model $g$ is less likely to be a derivative (infringing) product.

\textbf{Setup.} 
For a balanced evaluation, we sample 100 $(\mathcal{M}, g)$ instances across $\mathrm{TAM} \! \in \! \{0,1,2\}$ from the benchmark in Section~\ref{experiments}.
Model $g$ under the $\mathrm{TAM}\! = \!0$ setting is treated as clean, i.e., not infringing on any model in the candidate set, whereas the model in $\mathrm{TAM}\!>\!0$ is regarded as an infringing case. We apply MPS to compute the $p$-values with respect to $H_{0,\mathcal{M}}$, and report results averaged over the three runs.

\textbf{Results.} Figure~\ref{fig:case2} shows the averaged NI-Scores of both Semantic-MPS and Token-MPS across different lineage depths. We observe that the NI-Score remains consistently above 0.85 in non-infringement scenarios, i.e., $\mathrm{TAM}=0$, while it drops precipitously to near-zero values for all infringing scenarios where $\mathrm{TAM}>0$. This clear margin highlights the discriminative power of the NI-Score, confirming its utility as a potential indicator for automated non-infringement auditing.

\begin{figure}[]
    \centering
    
    \begin{subfigure}{0.49\linewidth}
        \centering
        \includegraphics[width=\linewidth]{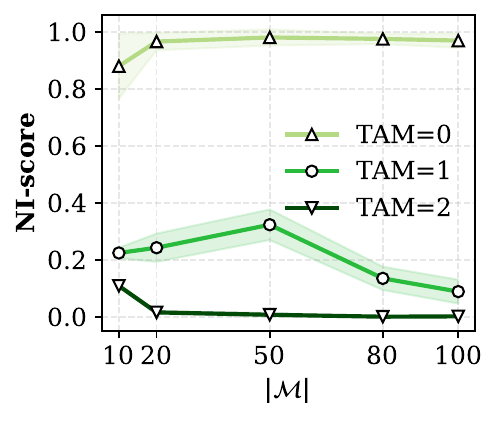}
    \end{subfigure}
    \hfill
    \begin{subfigure}{0.49\linewidth}
        \centering
        \includegraphics[width=\linewidth]{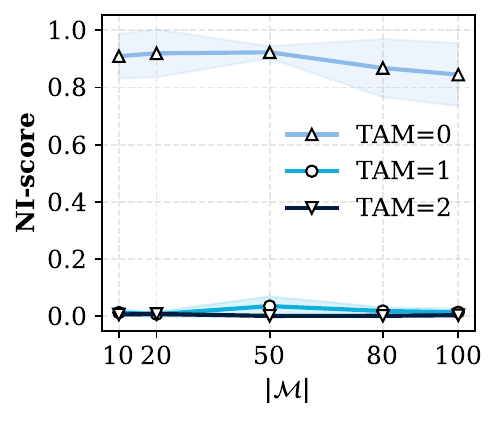}
    \end{subfigure}
    
 \caption{Analysis of averaged NI-Scores constructed by Semantic-MPS (left) and Token-MPS (right) on 100 multi-candidate test instances. $|\mathcal{M}|$ denotes the candidate set size. $\mathrm{TAM}$ indicates the number of true ancestor model among candidates. Higher scores signify that target model $g$ is more likely non-infringing.}
    \label{fig:case2}
\end{figure}

\section{Conclusion}
In this work, we first present the problem of model provenance testing with formal guarantees in multi-candidate settings. To address this, 
we propose Model Provenance Set (MPS), which leverages sequential significance testing to iteratively detect multiple potential provenances.
Furthermore, we establish the asymptotic theory of MPS to provide formal coverage guarantees for true provenance models at a predefined confidence level.
Extensive experiments validate that our method achieves reliable coverage and performs effectively across diverse model attribution and auditing tasks.
We hope this study can inspire future research on reliable model provenance and promote its adoption in practical IP protection and auditing scenarios.



\section*{Impact Statement}

This paper presents work whose goal is to advance the field of machine learning. There are many potential societal consequences of our work, none of which we feel must be specifically highlighted here.

\nocite{langley00}

\bibliography{A_reference}
\bibliographystyle{icml2026}

\newpage
\appendix
\onecolumn

\section{Related Work}
\label{apd:related_work}
Existing methods for model provenance detection primarily focus on verifying whether a model is derived from a specific copyrighted model. Depending on whether the copyrighted model is proactively modified, these methods are typically categorized into watermark-based and fingerprint-based approaches. 


\textbf{Watermarking.}  This technique attempts to embed identifiable features into the copyrighted model, which can be activated under specific conditions, allowing one to identify whether a suspect LLM is derived from it.
Some approaches introduce watermarks during training by modifying model parameters, for example, via fine-tuning on proprietary knowledge~\cite{li2023turning, xu2024hufu} or by exploiting controlled information in quantization strategies~\cite{li2023watermarking}.
However, such methods incur substantial retraining costs and may degrade model performance.
Consequently, other works introduce watermarking at inference time by integrating it into the text generation process.
These approaches primarily operate by transforming latent representations~\cite{zhang2024remark, peng2023you}, altering decoding rules~\cite{dathathri2024scalable, kuditipudi2024robust}, or modifying token-level logits~\cite{kirchenbauer2023watermark, liusemantic2024semantic, park2026watermod}, and typically require additional post-processing for watermark verification. While watermarking provides a proactive defense, it requires intrusive modifications during the pre-training or fine-tuning stages of the model. In contrast, our approach focuses on lightweight post-hoc provenance analysis, without requiring any prior intervention in the model training process.

\textbf{Fingerprinting.} In contrast, fingerprint-based provenance approaches extract unique inherent identifiers from models without modification. Depending on whether they have access to model internals, they can be categorized into white-box and black-box methods. 
\textbf{White-box methods} analyze internal structural properties. For example, REEF~\cite{zhang2025reef} identifies derivations by calculating centered kernel alignment (CKA) on intermediate-layer embeddings, while others leverage gradient features~\cite{wu2025gradient} or attention parameter distributions~\cite{yoon2025intrinsic}. Other works~\cite{hu2025fingerprinting,shi2025knowledge}, utilize response logits, bypassing the need for direct weight access.
\textbf{Black-box methods} instead rely on input–output behaviors under controlled or adversarial prompts. These include measuring token-level or semantic similarity of LLM outputs over either random sequences or stylistic queries~\cite{Nikolic2025Model, wu2025llmdnatracingmodel, ren2025cotsrf}, or eliciting characteristic behaviors using carefully designed prompts~\cite{gubri2024trap, pasquini2025llmmap, xu-etal-2025-ctcc} and the training data from parent models~\cite{kuditipudi2025blackbox}.
However, despite their versatility, most existing provenance methods focus on pairwise  model detection with empirical thresholds~\cite{hu2025fingerprinting} or learned classifiers~\cite{wu2025llmdnatracingmodel}, or adopt a winner-takes-all strategy that selects the single most similar model from multiple candidates~\cite{Nikolic2025Model}.
These heuristic decision rules lack rigorous statistical guarantees, inevitably limiting their reliability in complex, high-stakes lineage tracing scenarios.
In this work, we formalize provenance analysis by introducing rigorous statistical guarantees with provable theoretical validity.

\section{Connection to Model Confidence Set}
\label{apd:mcs_connection}

The model confidence set (MCS) framework, proposed by~\citet{Hansen2011mcs}, provides a principled approach to identifying a set of models that are statistically indistinguishable from the best-performing model.
The key idea is to test whether all models have equal expected losses using a relative loss-based test statistic.
For any two models $f_i, f_j \in \mathcal{M}$, the pairwise sample loss for observation $t$ is defined as $d_{ij,t} = L_{i,t} - L_{j,t}$, and the corresponding relative mean loss is given by $\bar{d}_{ij} = \frac{1}{N} \sum_{t=1}^{N} d_{ij,t}$.
The MCS procedure then sequentially eliminates candidate models that perform significantly worse than others until no such model remains, yielding the selected model set.

We adapt this idea to provenance analysis: instead of selecting models with the best predictive performance, we select models that are most similar to the target. 
Our proposed framework inherits the robust statistical foundation of the MCS, particularly its adaptive uncertainty measure and relative comparison mechanisms. In this spirit, our method acknowledges the difficulty of distinguishing related models from irrelevant ones by adjusting the set size based on the available evidence. When there is a clear separation in similarity signals, MPS yields a small, precise set. 
However, when the similarities of true sources and irrelevant models severely overlap (e.g., multiple non-provenance candidates form the same family as the true parent), MPS tends to include more likely candidates in predicted set, ensuring that the true source is not excluded.

Despite the similarity in the core statistical machinery,  our approach introduces fundamental innovations tailored for provenance analysis. First, while MCS assumes the existence of a unique best reference model, provenance analysis does not assume such a reference and may involve zero or multiple ancestral models. This shift necessitates a different test statistic that can identify provenance existence without requiring a pre-defined ground truth.
Second, MCS selects models with the smallest loss (best performance), whereas our framework selects models with the smallest distance to the target (highest similarity). This shift ensures that the procedure captures structural or behavioral lineage rather than mere task proficiency.
Finally, instead of eliminating the worst models, our procedure leverages the fact that all irrelevant models share indistinguishable similarities and should remain in the pool as a baseline. We therefore iteratively exclude the most similar model from the ensemble for testing. This mechanism allows MPS to handle multiple sources effectively or return an empty set when no candidate stands out from this background distribution.

\section{Proofs}
\label{apd:proofs}





\subsection{Proof of Theorem~\ref{thm:pvalue}}
\label{prf:pvalue}

\begin{proof}
We first establish the asymptotic normality of each studentized statistic $t_i$.
Define $Z_{i,t} \equiv d_{i\cdot,t} = L_{i,t} - \frac{1}{M}\sum_{j=1}^M L_{j,t}$, where $t = 1, \ldots, N$ indexes the prompts and $M = |\mathcal{M}|$.
By condition (a), the prompts are sampled independently, so $\{Z_{i,t}\}_{t=1}^N$ are independent for each model $f_i$.
To apply the Lyapunov CLT~\citep{dudley2002convergence}, we verify the Lyapunov condition: by condition (b), $\mathbb{E}|Z_{i,t}|^{2+\delta} < \infty$ for some $\delta > 0$, which implies
\[
\frac{1}{s_N^{2+\delta}} \sum_{t=1}^N \mathbb{E}|Z_{i,t} - \mathbb{E}[Z_{i,t}]|^{2+\delta} \to 0 \quad \text{as } N \to \infty,
\]
where $s_N^2 = \sum_{t=1}^N \mathrm{Var}(Z_{i,t}) = N \sigma_i^2$ and $\sigma_i^2 = \mathrm{Var}(Z_{i,t})$.
Then, we leverage the Lyapunov CLT and get 
\[
\sqrt{N} \bar{d}_{i\cdot} = \sqrt{N} \cdot \frac{1}{N} \sum_{t=1}^N Z_{i,t} \xrightarrow{d} \mathcal{N}(0, \sigma_i^2).
\]
Condition (c) ensures $\sigma_i^2 > 0$.
By the law of large numbers, the sample variance satisfies $\widehat{\mathrm{Var}}(\bar{d}_{i\cdot}) \overset{p}{\to} \sigma_i^2 / N$.
Combining this with the asymptotic normality of $\sqrt{N} \bar{d}_{i\cdot}$ and applying Slutsky's theorem yields
\[
t_i = \frac{\bar{d}_{i\cdot}}{\sqrt{\widehat{\mathrm{Var}}(\bar{d}_{i\cdot})}} \xrightarrow{d} \mathcal{N}(0, 1).
\]

The joint distribution of the statistic vector follows from the multivariate CLT.
Define $\psi_i(x_t) = d_{i\cdot,t} / \sigma_i$, so that $t_i = \frac{1}{\sqrt{N}} \sum_{t=1}^N \psi_i(x_t) + o_p(1)$.
Let $\mathbf{t} = (t_1, \ldots, t_M)^T$ and $\boldsymbol{\Psi}(x_t) = (\psi_1(x_t), \ldots, \psi_M(x_t))^T$.
The vector of studentized statistics admits the asymptotic linear representation
\[
\mathbf{t} = \frac{1}{\sqrt{N}} \sum_{t=1}^N \boldsymbol{\Psi}(x_t) + o_p(1).
\]
By the multivariate CLT, $\mathbf{t} \xrightarrow{d} \mathcal{N}(\mathbf{0}, \Sigma)$ for some covariance matrix $\Sigma$.
Under the null hypothesis $\mathrm{H}_{0,\mathcal{M}}: \mu_1 = \cdots = \mu_M$, all models have equal expected distance, so $\mathbb{E}[\psi_i(x_t)] = 0$ for all $i$.
Since the null hypothesis treats all models symmetrically, the covariance matrix $\Sigma$ is invariant under permutations of model indices.
It follows that $(t_{\pi(1)}, \ldots, t_{\pi(M)}) \xrightarrow{d} \mathcal{N}(\mathbf{0}, \Sigma)$ for any permutation $\pi$, establishing asymptotic exchangeability.

Finally, we verify the validity of the permutation $p$-value.
By standard results on permutation tests with asymptotically exchangeable statistics~\citep{Joseph1990Behavior}, the permutation distribution consistently approximates the null distribution as $R \to \infty$.
Since $T_{\min} = \min_i t_i$ is a continuous function of $\mathbf{t}$ (the minimum over finitely many coordinates is Lipschitz continuous), the continuous mapping theorem yields $T_{\min} \xrightarrow{d} \min_i Z_i$ where $\mathbf{Z} \sim \mathcal{N}(\mathbf{0}, \Sigma)$.
Let $F_0$ denote the CDF of this limiting distribution.
From the probability integral transform, $F_0(T_{\min}) \sim \mathrm{Uniform}(0, 1)$ under the null, which implies $p_{\mathcal{M}} \sim \mathrm{Uniform}(0, 1)$ asymptotically.
Therefore,
\[
\lim_{N, R \to \infty} \Pr(p_{\mathcal{M}} \le \alpha) \le \alpha.
\]
This completes the proof.
\end{proof}

\subsection{Proof of Theorem~\ref{thm:coverage}}
\label{prf:coverage}

Recall from Definition~\ref{definition:tam} that the true provenance set $\mathcal{M}^*$ consists of all models truly related to the target.
We formalize this as the set of models with the smallest expected distance to the target model:
\begin{equation*}
\mathcal{M}^* = \left\{ f \in \mathcal{M} : \mu_f = \min_{f' \in \mathcal{M}} \mu_{f'} \right\},
\end{equation*}
where $\mu_f = \mathbb{E}[L_{f,t}]$ is the expected distance of model $f$ as defined in Eq.~\eqref{eq:loss}.

\begin{proof}[Proof]
Let $f \in \mathcal{M}^*$ be any true provenance model.
By definition, $\mu_f = \min_{f' \in \mathcal{M}} \mu_{f'}$.
At iteration $k$, let $\mathcal{M}^{(k)}$ denote the current candidate set.
If $f \in \mathcal{M}^{(k)}$ and $\mathcal{M}^{(k)} \not\subseteq \mathcal{M}^*$, then there exists $f' \in \mathcal{M}^{(k)}$ with $\mu_{f'} > \mu_f$.
Hence $\mathrm{H}_{0, \mathcal{M}^{(k)}}$ is false.
By Theorem~\ref{thm:pvalue},
\[
\lim_{N,R \to \infty} \Pr(p_{\mathcal{M}^{(k)}} \le \alpha \mid \mathrm{H}_{0, \mathcal{M}^{(k)}}) \le \alpha.
\]
Upon rejection, the algorithm removes $f^* = \arg\min_{f_i \in \mathcal{M}^{(k)}} t_i$.
Since $\mu_f < \mu_{f'}$ for all $f' \in \mathcal{M}^{(k)} \setminus \mathcal{M}^*$, we have $\Pr(t_f < t_{f'}) \to 1$ as $N \to \infty$.
Thus, $f^* \in \mathcal{M}^*$, and $f^*$ is added to $\hat{\mathcal{M}}$.

The procedure terminates when $p_{\mathcal{M}^{(k)}} > \alpha$, at which point $\mathcal{M}^{(k)} \cap \mathcal{M}^* = \emptyset$ with probability at least $1 - \alpha$.
Therefore,
\[
\lim_{N,R \to \infty} \Pr(\mathcal{M}^* \subseteq \hat{\mathcal{M}}) \ge 1 - \alpha.
\]
\end{proof}

\subsection{Proof of Theorem~\ref{thm:efficiency}}
\label{prf:efficiency}

\begin{proof}
By the gap condition, for all $f \in \mathcal{M}^*$ and $f' \in \mathcal{M} \setminus \mathcal{M}^*$,
\[
\mu_f < \mu_{f'} - \delta.
\]
By the law of large numbers, $\bar{d}_{i\cdot} \overset{p}{\to} \mu_i - \bar{\mu}$ as $N \to \infty$, where $\bar{\mu} = \frac{1}{M}\sum_{j=1}^M \mu_j$.
It follows that for any $f \in \mathcal{M}^*$ and $f' \in \mathcal{M} \setminus \mathcal{M}^*$,
\[
\bar{d}_{f\cdot} - \bar{d}_{f'\cdot} \overset{p}{\to} \mu_f - \mu_{f'} < -\delta < 0.
\]
Since the test statistic $t_i = \bar{d}_{i\cdot} / \sqrt{\widehat{\mathrm{Var}}(\bar{d}_{i\cdot})}$ and $\widehat{\mathrm{Var}}(\bar{d}_{i\cdot}) \overset{p}{\to} \sigma_i^2 / N > 0$, we have $t_f < t_{f'}$ with probability approaching $1$.
Therefore, at each iteration $k$ where $\mathcal{M}^{(k)} \not\subseteq \mathcal{M}^*$, the model $f^* = \arg\min_{f_i \in \mathcal{M}^{(k)}} t_i$ satisfies $f^* \in \mathcal{M}^*$ with probability approaching $1$.
The algorithm terminates when $\mathcal{M}^{(k)} \subseteq \mathcal{M} \setminus \mathcal{M}^*$, at which point all models in $\mathcal{M}^*$ have been selected into $\hat{\mathcal{M}}$.

For the convergence rate, by Hoeffding's inequality, for bounded distance $L_{i,t} \in [0, 1]$,
\[
\Pr\left( |\bar{d}_{f\cdot} - \bar{d}_{f'\cdot} - (\mu_f - \mu_{f'})| > \delta/2 \right) \le 2\exp(-c N \delta^2)
\]
for some constant $c > 0$.
Since $\mu_f - \mu_{f'} < -\delta$, we have $\bar{d}_{f\cdot} - \bar{d}_{f'\cdot} < -\delta/2 < 0$ with probability at least $1 - 2\exp(-c N \delta^2)$.
Applying a union bound over at most $M^2$ pairs yields $\Pr(\hat{\mathcal{M}} = \mathcal{M}^*) \ge 1 - O(\exp(-c N \delta^2))$.
\end{proof}

\include*{Sections/related_work}





\section{Models and Benchmark}
\label{apd:benchmark} 
\paragraph{Model collection.} We collect model candidates for all provenance pairs from the Hugging Face (HF) platform by using their default crawling APIs~\cite{HuggingFace2023}.
To ensure coverage of widely used variants, we select the most popular models by download count, restricting model sizes to under 3B parameters and focusing on those compatible with PyTorch and Transformers.
Following prior works~\cite{Nikolic2025Model,wu2025llmdnatracingmodel}, we focus on the context of the proliferation of fine-tuned model derivations. Specifically, we consider the most reliable ground truth to be cases where LLMs explicitly specify their base model on the HF platform via the ``\texttt{base\_model:finetune:$<$base\_model\_name$>$}'' field in the model description, which covers common fine-tuning paradigms including full fine-tuning, knowledge distillation, and reinforcement learning from human feedback (RLHF). We exclude LLMs without this explicit indication and remove models that do not support English as indicated in their model cards.
Finally, we collect 455 LLMs spanning three derivation depths, comprising 383, 118, and 13 derivation chains of length 1, 2, and 3, respectively.
Note that a single long chain is decomposed into multiple overlapping segments of varying depths.
Figure~\ref{fig:bench_tree} shows examples of derivation relationships across three rounds of fine-tuning. Moreover, we list the top-20 ancestor models with the most derived models in Table~\ref{tab:models details}.

\paragraph{Test instance construction.} We construct test instances $(\mathcal{M},g)$ with varying candidate set sizes. To control the number of true provenance models in $\mathcal{M}$, we simply extend the established derivation chains to construct candidate pools with desired $\mathrm{TAM}$ counts. For example, to construct a test instance with $|\mathcal{M}|=10$ under $\mathrm{TAM}\!=\!2$ , we consider a derivation chain with 2 provenance models: $\mathrm{grandparent \ model} \to \mathrm{parent \ model}  \to g$, and we then augment the set with 8 additional models sampled from the remaining 455 models that are unrelated to $g$.

\begin{table}[t]
\caption{Top 20 pre-trained LLMs from the benchmark}
\label{tab:models details}
\resizebox{\textwidth}{!}{
\begin{tabular}{@{}l|lllc@{}}
\toprule
Hugging Face Model                        & Main Domain                        & \#Parameters (dtype) & License          & \#Derived Models \\ \midrule
Qwen/Qwen2.5-1.5B-Instruct                & General LLM / Instruct             & 1,543,714,304 (BF16) & Apache-2.0       & 28                          \\
deepseek-ai/DeepSeek-R1-Distill-Qwen-1.5B & Reasoning / Math                   & 1,777,088,000 (BF16) & MIT              & 21                          \\
meta-llama/Llama-3.2-1B-Instruct          & General LLM / Instruct             & 1,235,814,400 (BF16) & Llama3.2         & 17                          \\
google/gemma-2-2b                         & General LLM (Base)                 & 2,614,341,888 (F32)  & Gemma            & 15                          \\
google/gemma-2b                           & General LLM (Base)                 & 2,506,172,416 (BF16) & Gemma            & 14                          \\
Qwen/Qwen2.5-1.5B                         & General LLM (Base)                 & 1,543,714,304 (BF16) & Apache-2.0       & 13                          \\
openai-community/gpt2                     & General LLM (Base)                 & 137,022,720 (F32)    & MIT              & 12                          \\
Qwen/Qwen2.5-0.5B-Instruct                & General LLM / Instruct             & 494,032,768 (BF16)   & Apache-2.0       & 12                          \\
Qwen/Qwen2-1.5B                           & General LLM (Base)                 & 1,543,714,304 (BF16) & Apache-2.0       & 11                          \\
TinyLlama/TinyLlama-1.1B-Chat-v1.0        & General LLM / Chat                 & 1,100,048,384 (BF16) & Apache-2.0       & 11                          \\
Qwen/Qwen3-0.6B                           & General LLM (Base)                 & 751,632,384 (BF16)   & Apache-2.0       & 11                          \\
microsoft/Phi-3-mini-4k-instruct          & General LLM / Reasoning / Instruct & 3,821,079,552 (BF16) & MIT              & 10                          \\
deepseek-ai/deepseek-coder-1.3b-base      & Code LLM                           & 1,346,471,936 (BF16) & Deepseek & 10                          \\
Qwen/Qwen3-1.7B                           & General LLM (Base)                 & 2,031,739,904 (BF16) & Apache-2.0       & 9                           \\
microsoft/phi-2                           & General LLM / Reasoning            & 2,779,683,840 (F16)  & MIT              & 9                           \\
HuggingFaceTB/SmolLM2-135M                & General LLM (Small / Efficient)    & 134,515,008 (BF16)   & Apache-2.0       & 8                           \\
meta-llama/Llama-3.2-3B-Instruct          & General LLM / Instruct             & 3,212,749,824 (BF16) & Llama3.2         & 8                           \\
Qwen/Qwen3-0.6B-Base                      & General LLM (Base)                 & 596,049,920 (BF16)   & Apache-2.0       & 7                           \\
Qwen/Qwen2.5-Math-1.5B                    & Math / Reasoning LLM               & 1,543,714,304 (BF16) & Apache-2.0       & 7                           \\
google/gemma-3-1b-it                      & General LLM / Instruct             & 999,885,952 (BF16)   & Gemma            & 6                           \\ 
\ldots & \ldots           &    &        &                          \\ 
\end{tabular} }
\end{table}


\begin{figure}[t]
    \centering
    \includegraphics[width=0.9\linewidth]{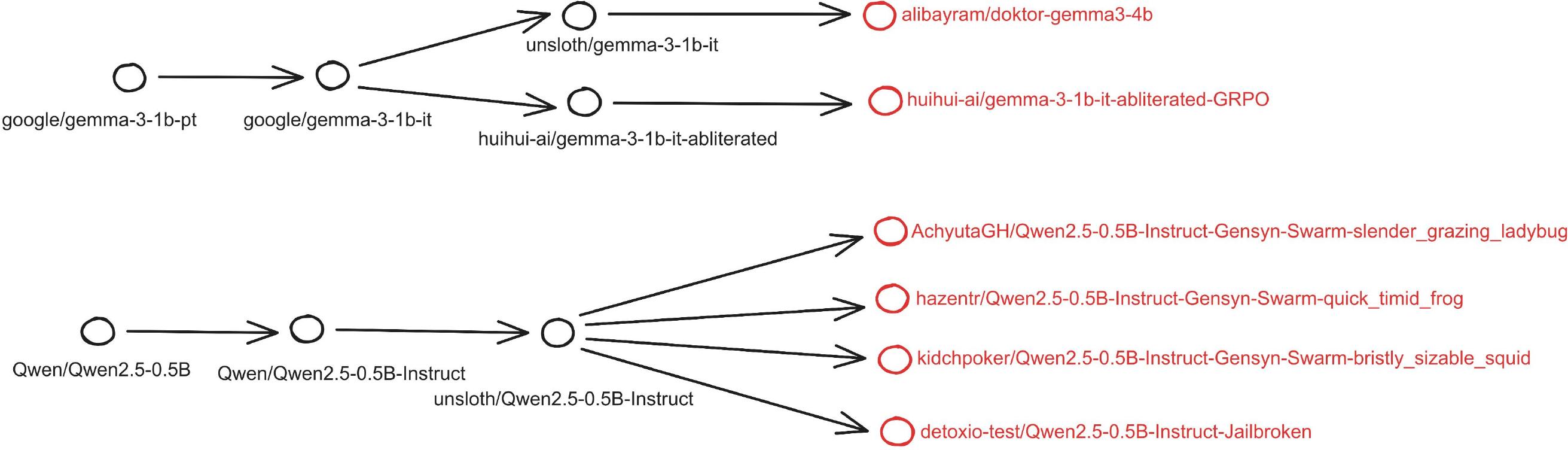}
    \caption{Examples of derivation chains with $\mathrm{TAM}\!=\!3$. Red nodes means the final child LLM versions after multiple fine-tuning.}
    \label{fig:bench_tree}
\end{figure}

\section{Implementation Details of MPS Procedure }
\label{apd:prompt_con}
\subsection{Distance Functions}
We adopt two simple yet effective metrics to measure pairwise model relationships. 
Both metrics require only query access to LLMs and are computationally efficient, making them suitable for large-scale provenance verification.

\paragraph{Next-token distance.}
The first metric is next-token similarity~\cite{Nikolic2025Model}, which evaluates whether two LLMs produce identical next tokens given the same prompt.
The authors reveal that fine-tuned derivative models are more likely to generate similar next-token outputs than unrelated models.
To incorporate this metric into MPS, we convert it into a distance measure by defining
\[
    L_{(f,g)_x} = 1 - \mathbf{1}\!\left(f(x), g(x)\right),
\]

where $f(x)$ and $g(x)$ denote the predicted next tokens of LLM-$f$ and LLM-$g$ on prompt $x$, respectively.

\paragraph{Response semantic distance.}
The second metric is response embedding similarity~\cite{wu2025llmdnatracingmodel}, which compares the semantic representations of model outputs generated from the same prompt.
Similarly, we adapt this metric to MPS by defining the corresponding distance as
\[
    L_{(f,g)_x} = 1 - \cos\!\left(f(x), g(x)\right),
\]
where $f(x)$ and $g(x)$ are the embeddings of the generated outputs from LLM-$f$ and LLM-$g$, respectively.

\paragraph{Usage in MPS procedure.}
For the next-token distance, given a prompt, we query each LLM to generate the next token taking the prompt as prefix input and compute the token-level discrepancy score accordingly.
For the response semantic distance, following~\citeauthor{wu2025llmdnatracingmodel}, we encode the generated one token of each prompt using Qwen3-Embedding-8B and use the first token embedding ($[\text{CLS}]$) from the final layer as the output representation. 

\subsection{Prompt Construction}
To produce diverse prompts for our provenance, we use the following ten powerful LLMs: 
\texttt{claude-3.5-sonnet}, \texttt{gpt-4o-mini}, \texttt{gemini-2.5-flash}, \texttt{kimi-dev-72b}, \texttt{grok-4-fast}, \texttt{z-ai/glm-4.6}, \texttt{qwen/qwen-plus-2025-07-28}, \texttt{meta-llama/llama-3.3-70b-instruct}, \texttt{deepseek/deepseek-chat-v3-0324}, and  \texttt{nousresearch/hermes-4-70b}.
We ask them to generate the short sequences with sufficient information(\textit{high-entropy}) as prompts, avoiding predictable continuations. 
For example, given a prefix such as “The first letter of the English alphabet is”, the next token (``A'') is almost deterministic under the language model. We interact with these LLMs via \texttt{OpenRouter API} using the instructions illustrated in Figure~\ref{fig:prompt_gen}. We generated 114,543 incomplete sentences, consuming 22.14\$ in total. By filtering out those whose word length is too short ($< 5$) or too long ($>20$), and then collecting 100,400 prompts.
\begin{figure}
    \centering
    \includegraphics[width=0.7\linewidth]{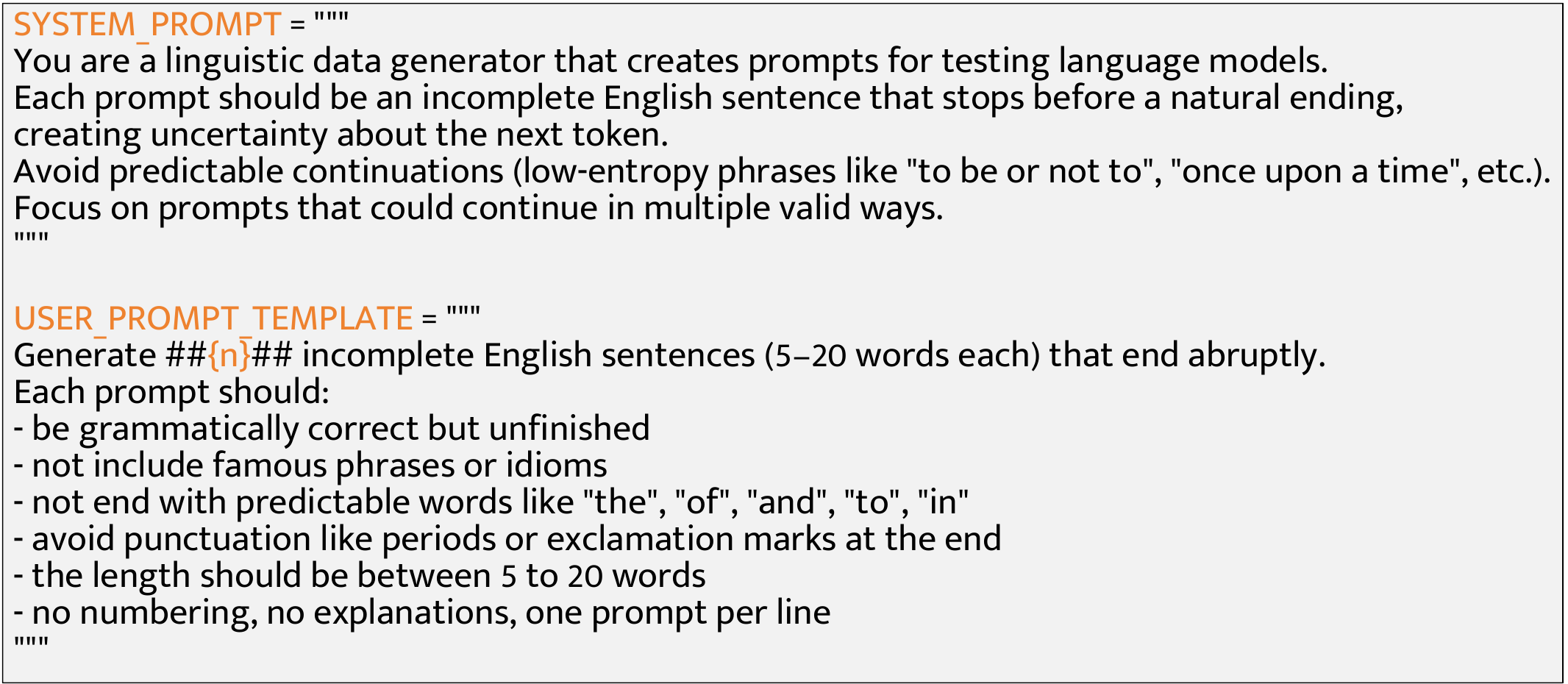}
    \caption{Instruction template for prompt generation}
    \label{fig:prompt_gen}
\end{figure}

\subsection{Experimental Environment} 
We run our model provenance testers on a Linux machine with 64-bit Ubuntu 22.04.5 LTS, 256GB RAM and $2\times$ 16-core AMD EPYC 9124 Processor @1.50 GHz and $8\times$ NVIDIA 4090 GPUs with 20GB RAM. All experiments are implemented using Python 3.10, CUDA 11.8, PyTorch 2.6 and Transformers 4.57.1 library.

\section{Extended Results of MPS}

\subsection{More Results of MPS on Specific $\alpha$}
\label{apd:main_res_alpha}

\begin{figure}[t]
    \centering
    \begin{subfigure}{\linewidth}
        \centering
        \includegraphics[width=\linewidth]{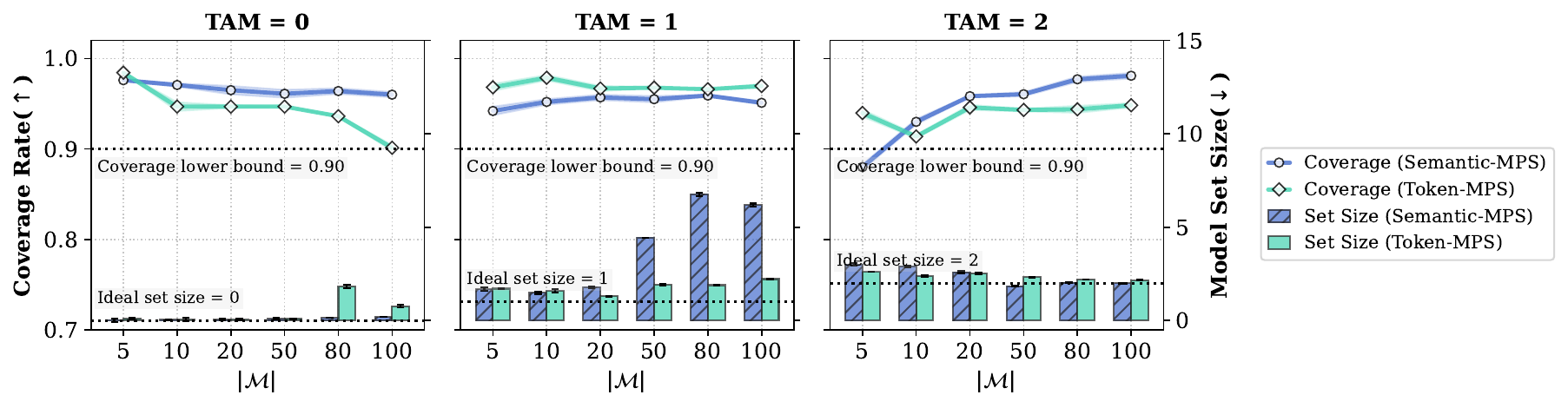}
    \end{subfigure}

    \begin{subfigure}{\linewidth}
        \centering
        \includegraphics[width=\linewidth]{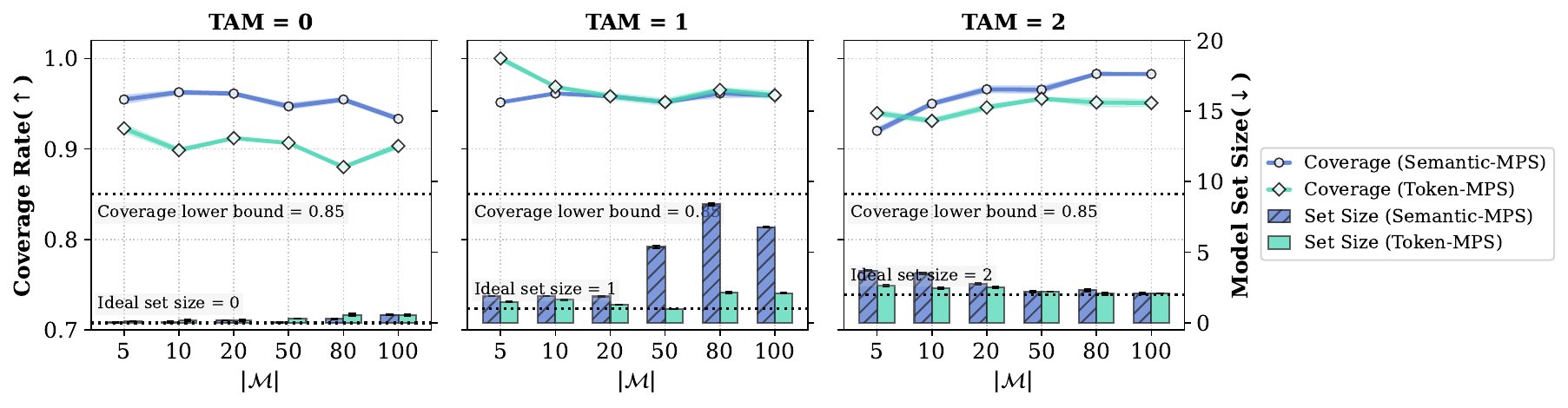}
    \end{subfigure}

     \caption{Averaged coverage and set size of MPS procedure with different score functions with $\alpha=0.1$ (Top) and $\alpha=0.15$ (Bottom). $|\mathcal{M}|$ denotes the candidate set size; $\mathrm{TAM}$ is the number of true ancestor models.}
    \label{fig:main_res_0.1}
\end{figure}

This subsection supplements a broader results of the MPS procedure under $\alpha\in \{0.1,0.15\}$ across different score functions and candidate sizes $|\mathcal{M}|\in \{5,10,20,50,80,100 \}$. 

As shown in the Figure~\ref{fig:main_res_0.1}, the results indicate that the MPS consistently achieves a coverage rate above the predefined lower bound. Notably, across different distance scoring functions, the method maintains a compact predicted set size that remains close to the ground-truth size and is substantially smaller than the full candidate set. For instance, in the case of Semantic-MPS under $\alpha=0.15$ with true 
$\mathrm{TAM}\!=\!1$ and a candidate size $|\mathcal{M}|=100$, the returned set size is consistently below 10, which is far less than the full candidate pool.

In summary, \textbf{MPS consistently preserves prescribed coverage guarantees while returning a compact subset of candidate models}, suggesting its practical utility for large-scale provenance auditing scenarios in enabling manual or secondary verification to be concentrated on a small number of potential provenance models.




\subsection{Analysis of Similarity Distributions}
\label{apd:score distribution analysis}

Comparing distance functions in Figures~\ref{fig:main_res} and Figure~\ref{fig:main_res_0.1}, we observe a performance crossover: while Token-MPS is competitive at shallow lineages ($TAM\!$ = 0, 1), Semantic-MPS demonstrates superior robustness as the lineage depth reaches $TAM\!=\!2$. \textbf{To investigate the underlying cause of this performance shift}, we analyze the similarity distributions across varying lineage depths. We sample 100 unrelated and lineage-related model pairs from $TAM\!=\!2$ derivation chains (e.g., $A \rightarrow B \rightarrow C$, where $B$ and $A$ represent one-hop and two-hop ancestors of $C$).

Figure~\ref{fig:score_dis} illustrates the comparative similarity distributions for Token and Semantic metrics, yielding several key insights:
(i) similarity decays monotonically with derivation depth, with models being most similar to their immediate ancestors and progressively less similar to more distant ones; (ii) despite this decay, lineage-related models remain separable from unrelated models within two generations; and (iii)
Token similarity provides clearer separation between directly related and unrelated pairs at shallow derivation depths, whereas cosine-based semantic similarity remains discriminative under deeper derivation chains, i.e., two-hop similarity exhibits limited overlap with unrelated model pairs, as token-level features drift due to repeated fine-tuning while semantic signals are better preserved.

Overall, the observed performance shift stems from token-level features drifting significantly during repeated fine-tuning, while semantic signals remain more stable. This motivates the development of robust scoring functions that can better capture underlying lineage continuity in complex provenance settings.

\begin{figure}
    \centering
    \begin{subfigure}{0.3\textwidth}
     \centering
        \includegraphics[width=\linewidth]{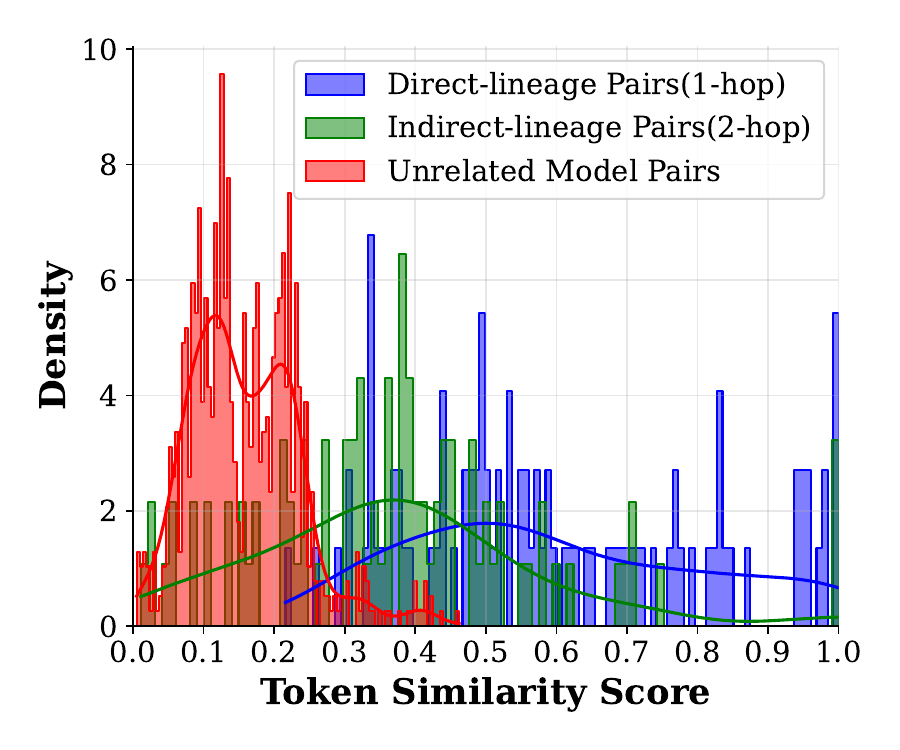}
    \end{subfigure}
    \begin{subfigure}{0.3\textwidth}
      \centering
    \includegraphics[width=\linewidth]{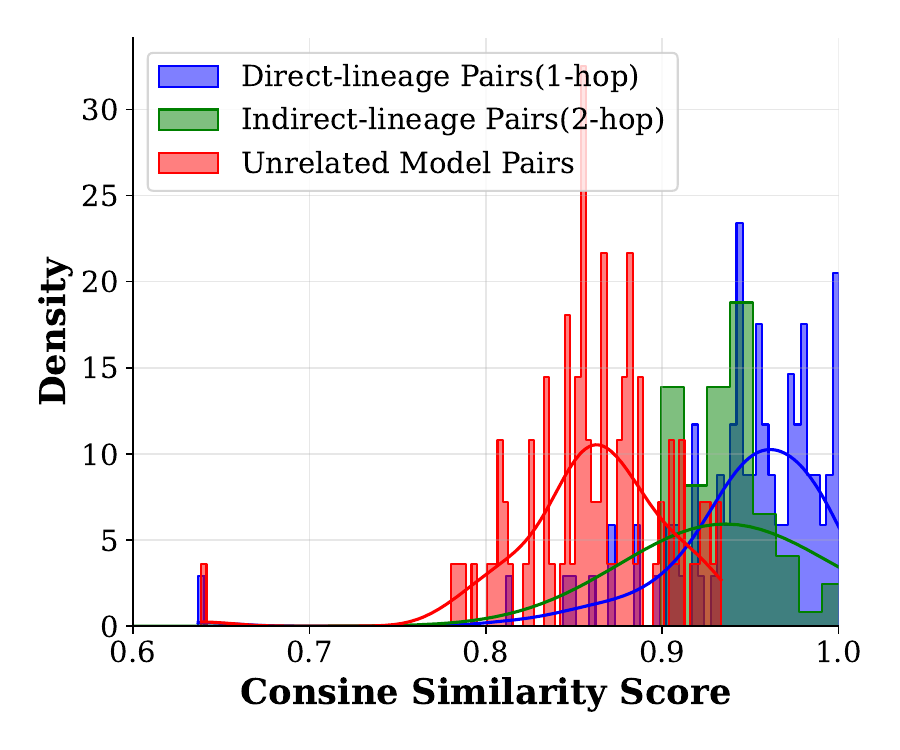}
     \end{subfigure}
     
  \captionof{figure}{Similarity score distributions for unrelated pairs (red) and lineage-related pairs at different derivation distances. \textit{1-hop} refers to parent–child models (blue), and \textit{2-hop} indicates grandparent–grandchild (green).}
\label{fig:score_dis}
\end{figure}

\subsection{Effect of Significance Level on Semantic-MPS}
\label{apd:semantic_alpha_ana}
Figure~\ref{fig:semtic_alpha} shows the results of Semantic-MPS under different significance levels. It can be observed that as $\alpha$ increases, its coverage rate consistently remains above the $1-\alpha$ boundary. Although the increase in $\alpha$ leads to an expected expansion of the prediction set, the sizes remain within acceptable levels. For example, even when $|\mathcal{M}|=100$, the set size does not exceed 10.
Overall, the behaviour of MPS is statistically valid and aligns strictly with the theoretical implications of $\alpha$.

\begin{figure}
    \centering
    \includegraphics[width=\linewidth]{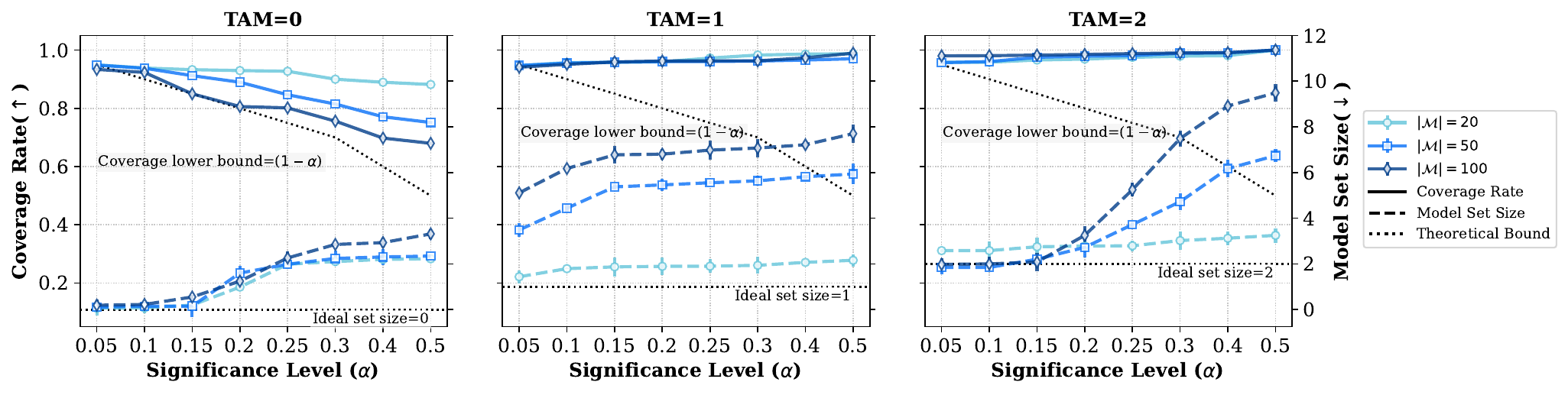}
    \caption{Average coverage rate and set size of Semantic-MPS with varying significance level $\alpha$. $|\mathcal{M}|$ denotes the candidate set size; $\mathrm{TAM}$ is the number of true ancestor models.}
    \label{fig:semtic_alpha}
\end{figure}

\section{Supplement of Empirical Applications}
\subsection{Baselines}
\label{apd: detail_pairwise_detection}


We compare MPS with the following provenance methods including two white-box approaches, i.e., REEF and LLMPrint and two black-box techniques, i.e., LLMDNA and FWER.

\textbf{REEF}~\cite{zhang2025reef} extracts intermediate hidden representations as model fingerprints and computes pairwise similarity using Centered Kernel Alignment (CKA). Following the default setting, we use the TruthfulQA dataset to extract the embeddings at layer 18 and then calculate the linear CKA score between two LLMs. As REEF does not perform provenance classification based on the CKA score, we simply adopt a threshold-based strategy based on empirical observations on the CKA scores over 281 unrelated and directly related pairs from our benchmark (excluding the overlapping pairs from Bench-A), where pairs with a score greater than 0.65 are classified as provenance.

\textbf{LLMPrint}~\cite{hu2025fingerprinting} optimizes the fingerprint triggers by the Greedy Coordinate Gradient (GCG) algorithm~\cite{zou2023universal} and computes a token preferences similarity by analyzing the logits of the first generated tokens between two LLMs. Following LLMPrint, we construct 300 fingerprint prompts per base LLM on the given token pairs, and use the fixed prefix “Randomly output a word from your vocabulary” followed by a suffix $s_{init}$ with 20 placeholders (``x'') for optimization. We run the GCG algorithm for 1,000 iterations to optimize a fingerprint prompt. Subsequently, provenance relationships are determined using the gray-box thresholding strategy proposed in LLMPrint with $z=1.64$.

\textbf{LLMDNA}~\cite{wu2025llmdnatracingmodel} extracts semantic representations for each LLM response and then trains a support vector machine (SVM) with an RBF kernel to classify the model pairs. We join the generated tokens per LLM into a single sequence by the whitespace and then use the Qwen3-Embedding-8B to encode them. During the training, we use $\mathrm{TAM}\!=\!0$ and $1$ to construct 281 negative and positive training pairs. Given a pairwise embedding ($\textbf{e}_f,\textbf{e}_g$), we concatenate them as the input representation $\textbf{z}$, i.e., $\textbf{z}=\text{concat}(\textbf{e}_f,\textbf{e}_g,|\textbf{e}_f-\textbf{e}_g|)$. To determine the hyper-parameters $C$ and $\gamma$ for the SVM, we perform a grid search over
$C \in [1, 10, 50, 100]$ and $\gamma \in [10^{-1}, 10^{-2}, 10^{-3}, 10^{-4}, 10^{-5}]$.
Based on this search, we select $C = 10$ and $\gamma = 10^{-3}$, while all other hyper-parameters are set to their default values.

\textbf{FWER}~\cite{Nikolic2025Model} employs Z-tests under family-wise error rate (FWER) control to determine whether a suspect parent model exhibits statistically significant token-level similarity to a target model, relative to a set of reference models that are known a priori to be unrelated to the target. For a set of parent candidates, FWER outputs the candidate most similar to the target model (if it passes the significance test) or returns null. We extend this by iteratively evaluating each candidate to construct a prediction set, maintaining a significance level of $\alpha\!=\!0.05$ and utilizing the same prompt set as MPS.

For each multi-candidate test instance $(\mathcal{M},g)$, we apply each baseline to iteratively determine the relationship of each candidate-target pair $(m,g)$ with $m\in \mathcal{M}$, and aggregate all candidates classified as provenance into the provenance set.

\begin{table}[]
\centering
\small
\caption{Averaged Performance for pairwise model provenance detection over three runs(mean $\pm$ std). The best results are bolded.} 
\label{tab:pairwise_mpt}
\begin{tabular}{@{}llll@{}}
\toprule
Method & Precision$\uparrow$ & Recall$\uparrow$   & Accuracy$\uparrow$ \\ \midrule
REEF          &  0.89 ± 0.11        &  0.83 ± 0.08       &  0.85 ± 0.03        \\ 
LLMPrint      &  0.92 ± 0.03        &  \textbf{0.97 ± 0.02}     &  0.95 ± 0.02            \\
LLMDNA        &  0.91  ± 0.18       &  0.95 ± 0.10        &  0.93 ± 0.11         \\
FWER          & 0.97 ± 0.01           & 0.89 ± 0.01         & 0.93 ± 0.04         \\
\midrule
Token-MPS     & 0.99 ± 0.00          & 0.92 ± 0.00         & 0.95 ± 0.00         \\
Semantic-MPS  & \textbf{0.99 ± 0.01}  & {0.96 ± 0.00} & \textbf{0.97 ± 0.00} \\ \bottomrule
\end{tabular}
\end{table}

\subsection{Model Provenance Detection}
\label{apd:pairwise_dectection}
In this subsection, we study pairwise provenance verification, i.e., determining whether a given pair of models exhibits a provenance relationship, which has been the primary focus of prior work~\cite{hu2025fingerprinting,wu2025llmdnatracingmodel}.
Although MPS is designed for set-valued inference, this can reduce to the case where the candidate set contains a single model ($|\mathcal{M}|=1$). Since MPS relies on relative distances across multiple candidates, we augment the candidate pool with known unrelated control models, following prior work~\cite{Nikolic2025Model}. 
MPS implements pairwise provenance inference by holding the control models in the candidate set during the exclusion phase and deciding whether to exclude the suspected parent model. Returning the suspected parent model indicates provenance; otherwise, non-provenance is indicated.

\textbf{Setup.} 
We utilize the Bench-A pairwise provenance dataset, which includes 10 hold-out control models to serve as independent references.
 We compare MPS with existing pairwise provenance detection methods, including white-box approaches such as REEF~\cite{zhang2025reef} and LLMPrint~\cite{hu2025fingerprinting}, as well as black-box techniques such as LLMDNA~\cite{wu2025llmdnatracingmodel} and FWER~\cite{Nikolic2025Model}. 
We report standard binary classification metrics, including Precision, Recall, and Accuracy.

\textbf{Results.} The experimental results on Bench-A are presented in Table~\ref{tab:pairwise_mpt}. Overall, MPS-based methods perform well in black-box settings, remaining competitive with white-box approaches that require internal access and carefully engineered features. Specifically, Semantic-MPS achieves the best overall performance, while Token-MPS also demonstrates strong and stable results, with all metrics exceeding 92\%. 
Among the baselines, FWER is the closest statistically grounded comparator. 
However, it performs pairwise hypothesis tests by sequentially comparing the similarity between the target and a suspect model against that of control models. This makes it difficult to distinguish weak provenance signals from control models,  resulting in reduced recall (e.g., 89\%).
By contrast, MPS jointly analyzes relative deviations among all candidates, allowing it to better distinguish true provenance models under marginal distance differences.





\end{document}